\documentclass[lettersize,journal]{IEEEtran}
\usepackage{amsmath,amsfonts}
\usepackage{algorithmic}
\usepackage{algorithm}
\usepackage{array}
\usepackage{xcolor}         
\usepackage[caption=false,font=normalsize,labelfont=sf,textfont=sf]{subfig}
\usepackage{textcomp}
\usepackage{stfloats}
\usepackage{url}
\usepackage{verbatim}
\usepackage{graphicx}
\usepackage{cite}

\usepackage{hyperref}       
\usepackage{url}            
\usepackage{booktabs}       
\usepackage{amsfonts}       
\usepackage{nicefrac}       
\usepackage{microtype}      
\usepackage{xcolor}         
\usepackage{graphicx}
\usepackage{amsmath}
\usepackage{float}
\usepackage{wrapfig}
\usepackage{caption}
\usepackage{subcaption}  
\usepackage{multirow}
\usepackage{comment}
\usepackage{caption}
\usepackage{orcidlink}

\newcommand\fable{\textit{FABLE}}
\newcommand\fableb{\textit{FABLE }}
\usepackage{amsthm}
\newtheorem{theorem}{Theorem}

\newtheorem{corollary}{Corollary}

\newtheorem{remark}{Remark}
\newtheorem*{theorem*}{Theorem}

\definecolor{deepred}{RGB}{225, 0, 0}
\definecolor{deepblue}{RGB}{0, 0, 225}

\begin{document}

\title{FABLE: A Localized, Targeted Adversarial Attack on Weather Forecasting Models}

\author{Yue Deng\orcidlink{0000-0002-8263-1871},
Asadullah Hill Galib\orcidlink{0000-0002-0686-4876},
Xin Lan\orcidlink{0000-0002-0607-2270},
Jack Gunn\orcidlink{0009-0007-3035-2444},
Pang-Ning Tan\orcidlink{0000-0003-3205-0339},
and Lifeng Luo\orcidlink{0000-0002-2829-7104}
\thanks{This research is supported by the U.S. National Science Foundation under grant IIS-2453100. Any use of trade, firm, or product names is for descriptive purposes only and does not imply endorsement by the U.S. Government. \textit{(Corresponding author: Pang-Ning Tan.)}}
\thanks{Yue Deng, Jack Gunn, and Pang-Ning Tan are with the Department of Computer Science \& Engineering, Michigan State University, East Lansing, MI 48824 USA (e-mail: \url{dengyue1@msu.edu}; \url{gunnjack@msu.edu}; \url{ptan@msu.edu}).}
\thanks{Asadullah Hill Galib is with the 
TSMC Technology Inc, San Jose, CA 95134 USA (e-mail: \url{asadgalib19@gmail.com}).}
\thanks{Xin Lan and Lifeng Luo are with the Department of Geography, Environment, and Spatial Sciences, Michigan State University, East Lansing, MI 48824 USA (e-mail: \url{lanxin1@msu.edu}; \url{lluo@msu.edu}).}
\thanks{The data and code are available at \url{https://github.com/yue2023cs/FABLE}.}
}



\maketitle

\begin{abstract}
Deep learning-based weather forecasting (DLWF) models have recently demonstrated significant performance gains over gold-standard physics-based simulation tools. However, these models are potentially vulnerable to adversarial attacks, which raises concerns about their trustworthiness. In this paper, we investigate the feasibility and challenges of applying existing adversarial attack methods to DLWF models and propose a novel framework called \textit{FABLE} (\underline{F}orecast \underline{A}lteration \underline{B}y \underline{L}ocalized targeted adv\underline{E}rsarial attack) to address them. FABLE performs a 3D discrete wavelet decomposition to disentangle the spatial and temporal components of the data. By regulating the magnitude of adversarial perturbations across different components, FABLE produces adversarial inputs that remain closely aligned with the original inputs while steering the DLWF models toward generating the targeted forecast outcomes. Experimental results on real-world weather datasets demonstrate the effectiveness of FABLE over baseline methods across various metrics. 
\end{abstract}

\begin{IEEEkeywords}
Spatiotemporal Learning, Weather Forecasting, Adversarial Robustness, Wavelet Transform.
\end{IEEEkeywords}

\section{Introduction}
Weather forecasting plays a crucial role in a wide range of human activities, influencing decision-making in agriculture, energy, insurance, and other sectors. With the growing impact of extreme weather events, 
more industries recognize the value of precise and timely weather information to mitigate risks and capitalize on opportunities. In recent years, deep learning-based weather forecasting (DLWF) models~\cite{gao2022earthformer,lin2022conditional,lam2023learning,bi2023accurate} have achieved significant improvements in prediction accuracy compared to traditional physics-based approaches. However, these models are susceptible to adversarial attacks \cite{chakraborty2018adversarial}, in which malicious actors can manipulate the forecasts by introducing subtle alterations to the input data, leading to incorrect predictions. For weather forecasting, such attacks could result in misguided management decisions, inefficient resource allocation, and inadequate preparedness for catastrophic weather events.

\begin{figure}[t!]
\centering
    \includegraphics[scale=0.36]{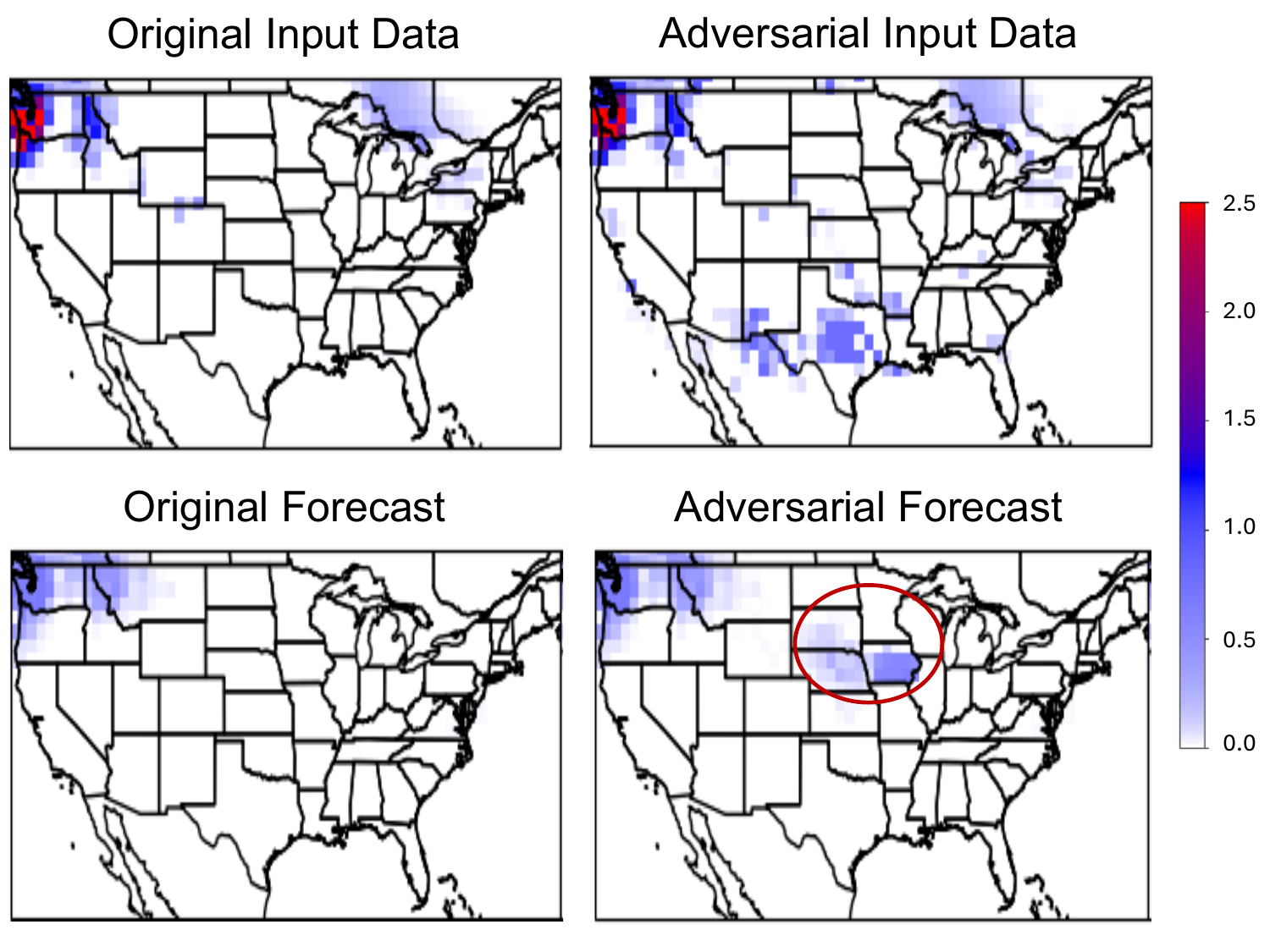}
    \caption{The top two panels show the original input data and its altered input (i.e., adversarial sample) produced by the {TAAOWPF} adversarial attack method~\cite{heinrich2024targeted} for the {NLDAS} precipitation dataset~\cite{mitchell2004multi}. The bottom two panels show their corresponding original and adversarial forecasts generated by the {CLCRN} weather forecasting model~\cite{lin2022conditional}. The red circle indicates the targeted region for forecast manipulation.}
    \label{fig:FABLE_case_study}
\end{figure}

This paper investigates the feasibility and challenges of designing \emph{localized, targeted adversarial attacks} on DLWF models. A localized targeted attack
involves malicious manipulation of the input data so that the model's forecasts for a pre-defined set of locations closely match the values desired by the attacker. For instance, consider an attacker who wants to influence agricultural markets by creating a false forecast of heavy rainfall in a major farming region. They could manipulate the input data of the DLWF model, causing it to predict significant rainfall when none is actually expected. Figure~\ref{fig:FABLE_case_study} shows an example of a tampered precipitation forecast for Iowa, generated by making alterations to the original input using the {TAAOWPF} adversarial attack method~\cite{heinrich2024targeted}. This false forecast could lead farmers to delay planting, alter irrigation schedules, or take unnecessary corrective actions based on incorrect predictions. The resulting market disruptions could benefit the attacker financially, especially if they have investments that would gain from such changes.  

For an attack to be successful, the adversarial inputs for DLWF models should satisfy the following criteria: (1) \textit{faithfulness}, ensuring that the adversarial inputs yield the intended forecasts at the targeted locations, with minimal changes to non-targeted locations, (2) \textit{closeness}, making the adversarial inputs closely resemble the original inputs to escape detection, and (3) \textit{geospatio-temporal realisticness}, preserving the dependencies within the original inputs to enhance stealthiness. However, as will be shown in this paper, managing the trade-offs among these criteria is challenging for existing adversarial attack methods \cite{goodfellow2014explaining,heinrich2024targeted,madry2017towards}. For instance, although the {TAAOWPF}~\cite{heinrich2024targeted} method successfully alters the forecast at the targeted locations in Figure~\ref{fig:FABLE_case_study}, the adversarial input it generates has discernible perturbations in non-targeted regions, 
making the attack more easily detectable. The added perturbations also modify the inherent spatial autocorrelations of the data.

To overcome this challenge, we propose a novel framework called \fableb (\underline{F}orecast \underline{A}lteration \underline{B}y \underline{L}ocalized targeted adv\underline{E}rsarial attack). Unlike conventional methods that directly perturb the raw input, FABLE introduces controlled perturbations of varying magnitudes across distinct components of the geospatio–temporal data, extracted by 3D discrete wavelet transform. We show that, by applying larger magnitudes of perturbations on high-frequency components, FABLE was able to preserve closeness and geospatio-temporal realisticness while achieving comparable faithfulness as other existing methods. 

\section{Related works}
\label{sec:Related works}

Weather forecasting has long been an active area of research due to its critical impact on our environment and society. Traditional models ~\cite{cote1998operational,skamarock2008description,hersbach2020era5} employ numerical simulations based on physical equations to predict future weather conditions. However, accurately modeling the chaotic nature of Earth's meteorological system remains a challenge. Towards this end, deep learning-based weather forecasting (DLWF) models have emerged as a powerful alternative for their ability to learn intricate geospatial–temporal patterns in data. These models, utilizing CNN~\cite{bi2023accurate,gao2022earthformer}, GCN~\cite{lin2022conditional,lam2023learning,price2025probabilistic},  Vision Transformer \cite{bojesomo2021spatiotemporal2}, Swin Transformer \cite{bojesomo2021spatiotemporal,bojesomo2023novel} and other neural architectures, have improved both single-step~\cite{galib2022deepextrema} and multi-step~\cite{lin2022conditional,lam2023learning} forecasts using univariate~\cite{lin2022conditional} and multivariate~\cite{lam2023learning} weather data.

As deep learning models become more prevalent, the risk of adversarial attacks on these systems has grown. Conventional adversarial attack techniques, such as FGSM~\cite{goodfellow2014explaining}, PGD~\cite{madry2017towards}, and MIM~\cite{dong2018boosting}, work by subtly perturbing the input data to mislead models into making incorrect predictions. Specifically, given a forecast model $g$, the adversarial attack would modify an \emph{input sample} $\mathbf{X}$ by adding a perturbation $\delta_{X}$, producing the \emph{adversarial sample} $\mathbf{X} + \mathbf{\delta_{X}}$. The goal of the attack is to ensure that the \emph{adversarial forecast} $g(\mathbf{X} + \mathbf{\delta_{X}})$ differs substantially from the \emph{original forecast} $g(\mathbf{X})$.

Adversarial attacks can be categorized by their objectives as \emph{untargeted}, \emph{semi-targeted}, or \emph{targeted} attacks~\cite{heinrich2024targeted}. Untargeted attacks aim to generate predictions that deviate substantially from the original forecast, i.e., $ \delta_\mathbf{X} = \arg \max_{\delta} \mathcal{L}\left[g(\mathbf{X}+\delta), g(\mathbf{X})\right]$, where $\mathcal{L}$ is the loss function. Semi-targeted attacks steer predictions into attacker-specified boundaries, while targeted attacks steer predictions toward a specific \emph{adversarial target} $\mathbf{\hat{Y}}^{'}$, i.e., $\delta_\mathbf{X} = \arg \min_{\delta} \mathcal{L}\left[g(\mathbf{X}+\delta), \mathbf{\hat{Y}}^{'}\right]$. Adversarial attacks can also be categorized based on attacker's level of access to the model~\cite{liu2022practical}---\emph{black-box} (access only to inputs and outputs), \emph{grey-box} (partial knowledge of architecture or training techniques), or \emph{white-box} (full knowledge of architecture, pre-trained parameters, and data).

In computer vision, adversarial attack methods would subtly alter input images or videos to produce misclassification or generation errors~\cite{chen2023imperceptible,wei2023efficient,wu2023imperceptible}. For spatial-temporal data, adversarial attacks have been studied in renewable energy forecasting and traffic prediction. Previous studies of energy forecasting~\cite{heinrich2024targeted,jiao2023gradient} primarily applied standard attack methods such as FGSM~\cite{goodfellow2014explaining} or PGD~\cite{madry2017towards} to perturb temporal weather inputs, while ignoring their spatial information~\cite{heinrich2024targeted,ruan2023vulnerability,jiao2023gradient}. For traffic forecasting, existing methods mainly focus on developing untargeted attacks that disrupt the overall traffic flow predictions, aiming to create congestion~\cite{zhu2023adversarial,liu2022practical}. 

Adversarial attacks on DLWF models have become an emerging research topic in recent years. Chichifoi et al. \cite{chichifoi2025evaluating} examined the vulnerability of temperature forecasting models under data poisoning attacks in a federated learning setting. Their findings revealed that corrupting even a small subset of clients can distort predictions across large spatial regions. 
Imgrund et. al. \cite{imgrund2025adversarial} proposed an adversarial attack method designed for DLWF models built on autoregressive diffusion architectures, such as GenCast~\cite{price2025probabilistic}. 
Their approach perturbs the input for specific regions to intentionally increase forecast errors, generating adversarial inputs that can trigger spurious predictions of non-existent extreme events. However, the approach is applicable to diffusion-based DLWF models only. 
In addition, some existing approaches \cite{imgrund2025adversarial, arif2025forecasting} restrict the perturbations to certain patched regions, which may introduce discontinuities between the perturbed areas and the surrounding meteorological fields, thereby altering the spatio-temporal autocorrelations present in the original input data. Unlike prior works that primarily aim to induce arbitrarily large forecast errors, our study seeks to drive the forecasts at certain locations toward specific target values while preserving closeness and geospatio-temporal realisticness of the adversarial inputs.


\section{Preliminaries}
\label{Preliminaries}

Let $\mathcal{D}=(\mathbf{Z}_1\mathbf{Z}_2\cdots\mathbf{Z}_t\cdots)$ be a geospatio-temporal dataset, where $\mathbf{Z}_t\in \mathbb{R}^{r \times c}$ denote the weather observations for $r \times c$ gridded locations at time step $t$. We further denote $Z_{tij}$ as the value of the weather variable at time step $t$ at a given location, whose latitude and longitude are indexed by $(i,j)$, where $i \in \{1,2,\cdots,r\}$ and $j \in \{1,2,\cdots,c\}$. At each time step $t_0$, we construct a pair of tensors:  (1) $\mathbf{X}(t_0) \in \mathbb{R}^{(\alpha+1) \times r \times c}$, a predictor (look-back) window of length $\alpha+1$ containing weather observations for all locations at $t_0$ and its preceding $\alpha$ time steps; and (2) $\mathbf{Y}(t_0) \in \mathbb{R}^{\beta \times r \times c}$, a forecast window of length $\beta$ containing the observations for all locations over the subsequent $\beta$ time steps. Thus, $X_{\tau ij}(t_0) = Z_{t_0-\alpha-1+\tau,ij}$, where $\tau \in \{1, 2, \cdots, \alpha+1\}$, and $Y_{\tau ij}(t_0) = Z_{t_0+\tau,ij}$, where $\tau \in \{1, 2, \cdots, \beta\}$. For notational convenience, we abbreviate $\mathbf{X}(t_0)$ as $\mathbf{X}$ and $\mathbf{Y}(t_0)$ as $\mathbf{Y}$.

\subsection{Problem Statement}

Consider a DLWF model $g$ that produce a \textbf{\textit{forecast}} $\mathbf{\hat{Y}} \in \mathbb{R}^{\beta \times r \times c}$ given a \textbf{\textit{predictor}} $\mathbf{X} \in \mathbb{R}^{(\alpha+1)\times r \times c}$, i.e., $\mathbf{\hat{Y}}=g(\mathbf{X})$. Let
$S_{\mathbf{\hat{Y}}}=\left\{(\tau, i,j) \ | \ \tau \in \Gamma; (i,j) \in \Omega \right\}$ be the \textbf{\emph{localized attack domain}}, where $\Gamma = \{1,\cdots,\beta\}$ and $\Omega = \{1, \cdots, r\} \times \{1, \cdots, c\}$ are the temporal and spatial domains. A \textbf{\emph{localized adversarial target}} is defined as $\mathbf{\hat{Y}}'=\mathbf{\hat{Y}} + \delta_{\mathbf{\hat{Y}}}$, where 
$\delta_{\hat{Y}_{\tau ij}} \neq 0$ for $(\tau, i, j) \in S_{\mathbf{\hat{Y}}}$ and zero elsewhere. 

Let $\mathbf{X}'$ be the \textbf{\textit{adversarial predictor}} and $g(\mathbf{X}')$ be its corresponding \textbf{\textit{adversarial forecast}} generated by $g$. 

\vspace{0.3cm}
\noindent\fbox{\begin{minipage}{0.97\linewidth}
\textbf{Goal:} Given the original predictor $\mathbf{X}$ and a localized adversarial target $\hat{\mathbf{Y}}'$,
our goal is to generate an adversarial predictor $\mathbf{X'}$ that manipulates the model $g$'s adversarial forecast $g(\mathbf{X'})$ in a way that meets the following closeness, faithfulness, and geospatio-temporal realisticness criteria. 
\end{minipage}}
\vspace{0.3cm}


\begin{table}[t!]
\caption{Summary of key notations}
\label{tab:notations}
\centering
\begin{tabular}{p{2.4cm} p{5.7cm}}
\hline
Notation & Description \\
\hline
$g(\cdot)$ & Deep learning weather forecasting (DLWF) model \\
$\alpha + 1$ & Length of predictor window \\
$\beta$ & Length of forecast window \\
$\mathbf{X} \in \mathbb{R}^{(\alpha+1)\times r\times c}$ 
& Original predictor over $r \times c$ locations across $\alpha+1$ time steps\\
$\mathbf{X}' \in \mathbb{R}^{(\alpha+1)\times r\times c}$ 
& Adversarial predictor \\
$\mathbf{\delta_X} \in \mathbb{R}^{(\alpha+1)\times r \times c}$ 
& Perturbation added to $\mathbf{X}$, \textit{i.e.}, $\mathbf{\delta_X}=\mathbf{X}'-\mathbf{X}$ \\
$\epsilon$ 
& Perturbation bound on $\mathbf{\delta_X}$, \textit{i.e.}, $\|\mathbf{\delta_X}\|\le \epsilon$ \\
$\hat{\mathbf{Y}} \in \mathbb{R}^{\beta\times r\times c}$ 
& Original forecast, \textit{i.e.}, $\hat{\mathbf{Y}}=g(\mathbf{X})$, over $r \times c$ locations across $\beta$ time steps \\
$\hat{\mathbf{Y}}' \in \mathbb{R}^{\beta\times r\times c}$ 
& Adversarial target, \textit{i.e.}, $\mathbf{\hat{Y}}'=\mathbf{\hat{Y}}+\delta_{\mathbf{\hat{Y}}}$ \\
$g(\mathbf{X}') \in \mathbb{R}^{\beta \times r \times c}$ & Adversarial forecast\\
$\mathbf{C}^{\mathbf{f}}\in\mathbb{R}^{\frac{\alpha+1}{2}\times \frac{r}{2}\times \frac{c}{2}}$ 
& 3D wavelet coefficients of $\mathbf{X}$ for each sub-band $\mathbf{f}$ \\
$\mathbf{\omega}^{\mathbf{f}}$ 
& Penalty weights for perturbations added to $\mathbf{C}^{\mathbf{f}}$ \\
\hline
\end{tabular}
\end{table}

\noindent{\textbf{Closeness. }} This criterion requires the adversarial predictor $\mathbf{X'}$ to remain close to the original predictor $\mathbf{X}$ to escape detection:
\begin{equation}
\textrm{Closeness} \equiv \|\delta_{\mathbf{X}}\|_1=\|\mathbf{X'}-\mathbf{X}\|_1,
\label{eqn:closeness}
\end{equation}
where $\delta_{\mathbf{X}}$ denotes the perturbation applied to $\mathbf{X}$. Most adversarial attack methods further constrain the perturbation magnitude by bounding it with a factor $\epsilon$ to preserve closeness. 


\vspace{0.2cm}
\noindent{\textbf{Faithfulness. }} This criterion ensures that the adversarial forecast $g(\mathbf{X'})$ aligns with the adversarial target $\hat{\mathbf{Y}}'$. 
For targeted attacks, we further distinguish between \textbf{in-target} and \textbf{out-target faithfulness} criteria. 
Given a localized attack domain, $S_{\mathbf{\hat{Y}}}$, 
we quantify the in-target faithfulness of the adversarial forecast $g(\mathbf{X'})$ using the following metric: 
\begin{equation}
\textrm{In-AE} \equiv \|g(\mathbf{X'})_\mathbf{in} - \mathbf{\hat{Y}}'_{\mathbf{in}}\|_1, 
\label{eqn:in-ae}
\end{equation}
where $\mathbf{\hat{Y}'_{in}}=\{\hat{Y}'_{\tau ij}|(\tau, i,j)\in S_{\mathbf{\hat{Y}}}\}$ and $g(\mathbf{X'})_\mathbf{in}=\{g(\mathbf{X'})_{\tau ij}|(\tau,i,j) \in S_{\mathbf{\hat{Y}}}\}$ are the corresponding adversarial target and adversarial forecast for locations within the attack domain. Analogously, we measure the out-target faithfulness of the adversarial forecast $g(\mathbf{X'})$ as follows: 
\begin{equation}
\textrm{Out-AE} \equiv \|g(\mathbf{X'})_\mathbf{out} - \mathbf{\hat{Y}} _\mathbf{out}'\|_1, 
\label{eqn:out-ae}
\end{equation}
where $\mathbf{\hat{Y}}'_{\mathbf{out}}=\{\hat{Y}_{\tau ij}'|(\tau, i,j)\not\in S_{\mathbf{\hat{Y}}}\}$ and $g(\mathbf{X'})_\mathbf{out}=\{g(\mathbf{X'})_{\tau ij}|(\tau, i,j) \not\in S_{\mathbf{\hat{Y}}}$ are the adversarial target and adversarial forecast for locations outside the attack domain. 

\vspace{0.2cm}
\noindent{\textbf{Geospatio-Temporal Realisticness. }}
This criterion states that the adversarial predictor $\mathbf{X}'$ should preserve the spatial and temporal autocorrelations of the original predictor $\mathbf{X}$. We measure it using the following metrics:

\begingroup\small
\small
\begin{align}
    R_S(\mathbf{X}',\mathbf{X})&=\frac{1}{\alpha+1}\sum_{\tau=1}^{\alpha+1}|I(\mathbf{X}_{\tau}')-I(\mathbf{X}_{\tau})|; 
    \label{eq:spatial_smoothness} \\
    R_T(\mathbf{X}',\mathbf{X}) &= \frac{1}{r c}\sum_{i=1}^r \sum_{j=1}^{c} \frac{1}{\alpha+1} \sum_{l=1}^{\alpha+1} |\rho_l(\mathbf{X}_{ij}^{'})-\rho_l(\mathbf{X}_{ij})|,
    \label{eq:temporal_smoothness}
\end{align}
\normalsize
\endgroup
where 
\begingroup\small
$$
    I(X_{\tau})=\frac{r^2\times c^2}{W}\frac{\sum_{(i,j)(k,l)}w_{ij,kl}(X_{\tau ij}-\mathbf{\overline{X}}_{\tau})(X_{\tau kl}-\mathbf{\overline{X}}_{\tau})}{\sum_{(i,j)}(X_{\tau ij}-\mathbf{\overline{X}}_{\tau})^2}
$$ 
\endgroup

\noindent is the \textit{Moran's I} metric \cite{moran1950notes}, which quantifies the spatial autocorrelation within a spatial map $\mathbf{X}_{\tau} \in \mathbb{R}^{r \times c}$, and 

\begingroup\small
$$
    \rho_l(X_{ij}) = \frac{\sum_{\tau=1}^{\alpha+1-l} (X_{\tau ij} - \mathbf{\overline{X}}_{ij})(X_{\tau+l,ij} - \mathbf{\overline{X}}_{ij})}{\sum_{\tau=1}^T (X_{\tau ij} - \mathbf{\overline{X}}_{ij})^2}
$$
\endgroup
is the lag-$l$ temporal autocorrelation \cite{brockwell2002introduction} of the time series at location $(i,j)$. For Moran's I, $\mathbf{\overline{X}}_{\tau} \in \mathbb{R}$ is the average value of the spatial map $\mathbf{X}_{\tau}$ and $W=\sum_{(i,j),(k,l)}w_{ij,kl}$ is the sum over an $(r \times c) \times (r \times c)$ weight matrix representing the spatial dependencies between locations $(i,j)$ and $(k,l)$. 
Specifically, $\omega_{ij,kl} = 1/d_{ij,kl}$, where $d_{ij,kl}$ is the geodesic distance between the two locations. For temporal autocorrelation, $\mathbf{\overline{X}}_{ij} \in \mathbb{R}$ is the average value of time series $\mathbf{X}_{ij} \in \mathbb{R}^{\alpha+1}$ at location $(i,j)$. 

This study considers a white-box attack scenario, wherein the adversary has the full knowledge of $g$. This assumption is justified, as most popular DLWF models 
have publicly released their architectures and checkpoints, such as {GraphCast}~\cite{lam2023learning} and {GenCast}~\cite{price2025probabilistic} by Google DeepMind, {FourCastNet}~\cite{pathak2022fourcastnet} by NVIDIA, {Aurora}~\cite{bodnar2025foundation} by Microsoft Research, and {Prithvi-WxC}~\cite{schmude2024prithvi} by IBM and NASA, among others.
The white-box attacks also facilitate a principled assessment of worst-case robustness by enabling the derivation of theoretical lower bounds~\cite{bhagoji2019lower, carlini2019evaluating}, and have been employed in recent works~\cite{chen2025towards}. Extension to black-box scenario is deferred to future work. 

\subsection{Wavelet Transform}
\label{sec:wavelet}
\textit{Wavelet transform}~\cite{debnath2015wavelet} enables the multiresolution analysis of a signal by decomposing it into its underlying components at multiple scales. The multi-level decomposition of a one-dimensional signal $X_t$ can be expressed as a linear combination of its low-frequency component at initial scale $j_0$ and higher-frequency details at finer scales $j=\{j_0, j_0+1,\cdots\}$:

\begingroup\small
\begin{equation*}
X_t = \sum_k C_k^L(j_0) 2^{j_0/2} \phi(2^{j_0}t - k) + \sum_{j=j_0}^\infty \sum_k C_k^H(j) 2^{j/2} \psi(2^j t - k),
\end{equation*}
\endgroup
where $C_k^L(j_0)$ and $ C_k^H(j)$ are the \textit{approximation} and \textit{detail} coefficients, respectively. The index $k$ is a translation parameter that determines the spatial or temporal shifts of the wavelet basis functions to ensure that $\phi(2^{j_0}t - k)$ and $\psi(2^j t - k)$ are properly positioned to localize the representation of $X_t$ at different scales. 
The function $\phi(t)$, known as the \textit{scaling function}, is responsible for approximating the low-frequency components of $X_t$, while $\psi(t)$, known as the \textit{wavelet function}, captures its high-frequency variations and localized details. 

We consider the decomposition of a signal $X_t$, where $t \in \{1, 2, \cdots, T\}$ and $T$ is assumed to be a multiple of 2 (with appropriate padding if necessary). The wavelet expansion of $X_t$ at the initial scale $j_0 = 0$ with the fine scale $j = 0$ for a one-level decomposition is given by
\begin{align}
\label{eq:one_level_wavelet_decomposition}
X_t = \sum_k C_k^L \, \phi(t - k) + \sum_k C_k^H \, \psi(t - k),
\end{align}
where $C_k^L \equiv C_k^L(0)$ and $C_k^H \equiv C_k^H(0)$.


The \textit{Haar wavelet} is a commonly used wavelet basis with the following scaling and wavelet functions: 
\begingroup
$$
\phi(z) = 
\begin{cases} 
1, & 0 \leq z < 1,\\
0, & \text{otherwise.}
\end{cases},
\quad
\psi(z) = 
\begin{cases} 
1, & 0 \leq z < 0.5, \\
-1, & 0.5 \leq z < 1, \\
0, & \text{otherwise.}  
\end{cases}  
$$
\endgroup

By dilating the basis functions from their initial support range of $[0,1)$ to $[0,2)$, translating the interval to $[k-1,k+1)$, and normalizing by $\frac{1}{\sqrt{2}}$ to ensure unit $\ell_2$-norm, the resulting orthonormal basis functions are given by

\vspace{0.2cm}
\begin{tabular}{l}
$\phi(z)=\frac{1}{\sqrt{2}}\,\phi\!\left(\frac{z-(k-1)}{2}\right) =
\begin{cases}
\frac{1}{\sqrt{2}}, & k-1 \le z < k+1\\
0, & \text{otherwise}
\end{cases}$ \\
\\
$\psi(z)=\frac{1}{\sqrt{2}}\,\psi\!\left(\frac{z-(k-1)}{2}\right) =
\begin{cases}
\frac{1}{\sqrt{2}}, & k-1 \le z < k\\
-\frac{1}{\sqrt{2}}, & k \le z < k+1\\
0, & \text{otherwise}.
\end{cases}$
\end{tabular}
\vspace{0.2cm}

\noindent where $k \in \left\{1,2,\dots,\frac{T}{2}\right\}$ is a dyadic pair associated with the time steps $t = 2k-1$ or $t = 2k$. Substituting the two time steps into Equation~(\ref{eq:one_level_wavelet_decomposition}) and using the $\phi(z)$ and $\psi(z)$ functions defined above, we obtain:
\begingroup
$$
X_{2k-1} = \frac{C_k^L}{\sqrt{2}} + \frac{C_k^H}{\sqrt{2}}  
\ \ \ \ \textrm{and} \ \ \ \ 
X_{2k} = \frac{C_k^L}{\sqrt{2}} - \frac{C_k^H}{\sqrt{2}}.$$
\endgroup
The approximation coefficients $C_k^L$ and detail coefficients $C_k^H$ can be derived from the preceding equations as
\begingroup
$$
C^{L}_k = \frac{X_{2k-1} + X_{2k}}{\sqrt{2}}
\ \ \ \  \textrm{and} \ \ \ \
C^H_k = \frac{X_{2k-1} - X_{2k}}{\sqrt{2}}.
$$
\endgroup



\begin{figure}[t!]
	\centering
	\includegraphics[width=3.53in]{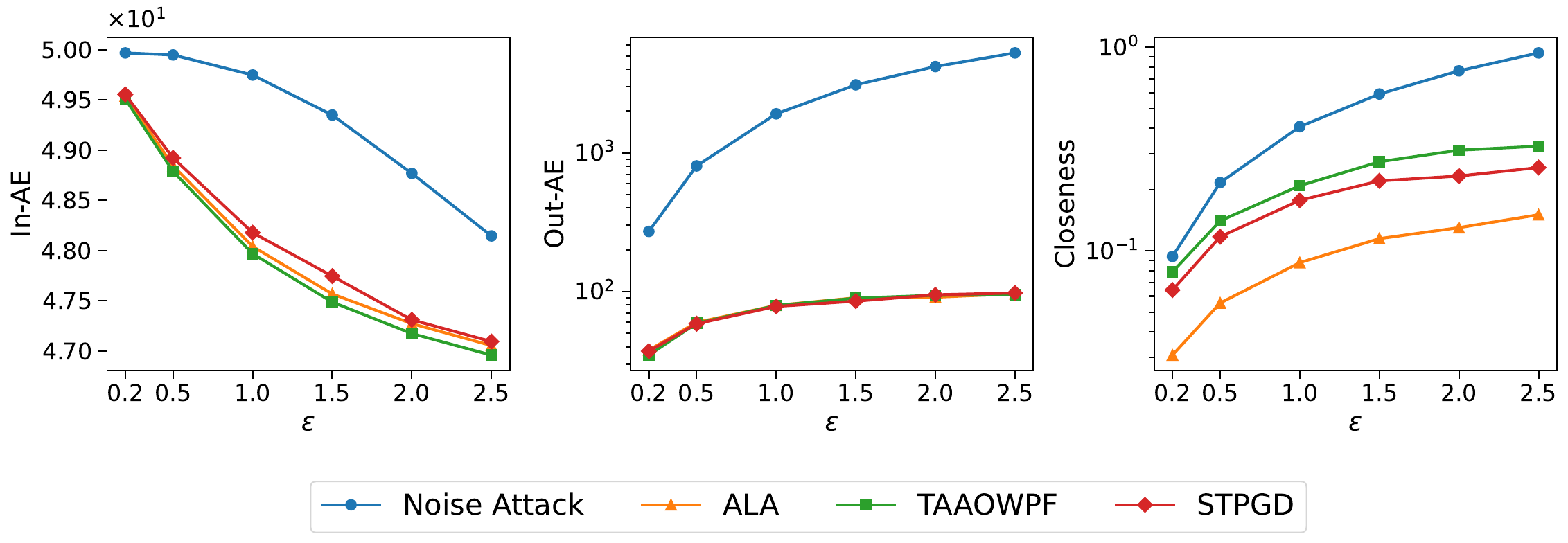}
	\caption{
    Performance comparison of existing adversarial attack methods on the {CLCRN} model~\cite{lin2022conditional} for {NLDAS} precipitation data~\cite{mitchell2004multi}. The x-axis denotes the perturbation bound $\epsilon$ and the y-axis denotes faithfulness ({in-AE} and {out-AE}) and closeness. 
    }
	\label{fig:baselines_results}
\end{figure}

\section{Attack Feasibility on DLWF Models}
\label{sec:Feasibility of Adversarial Attacks on Weather Forecasting Models}



We first investigate the feasibility of applying 3 existing attack methods---{ALA}~\cite{ruan2023vulnerability}, 
{TAAOWPF}~\cite{heinrich2024targeted}, and {STPGD}~\cite{liu2022practical}---on the {CLCRN} model~\cite{lin2022conditional}. The performance of these attacks are evaluated on the {NLDAS} daily precipitation and temperature datasets~\cite{mitchell2004multi} from 2020 to 2023. 
Each dataset contains weather observations for $1{,}320$ gridded locations across North America at $1^{\circ} \times 1^{\circ}$ resolution. {ALA}~\cite{ruan2023vulnerability} employs the Adam optimizer while {TAAOWPF}~\cite{heinrich2024targeted} uses the projected gradient descent (PGD) algorithm to learn the adversarial predictor to produce the adversarial target. {STPGD}~\cite{liu2022practical} extends the PGD-based method by selectively perturbing a subset of locations whose gradients have the most impact on the loss. We also consider the {Noise Attack}, which injects random Gaussian noise into the original predictor and selects the adversarial one whose forecast is closest to the target. 

\begin{figure}[t!]
    \centering
    \includegraphics[scale=0.35]{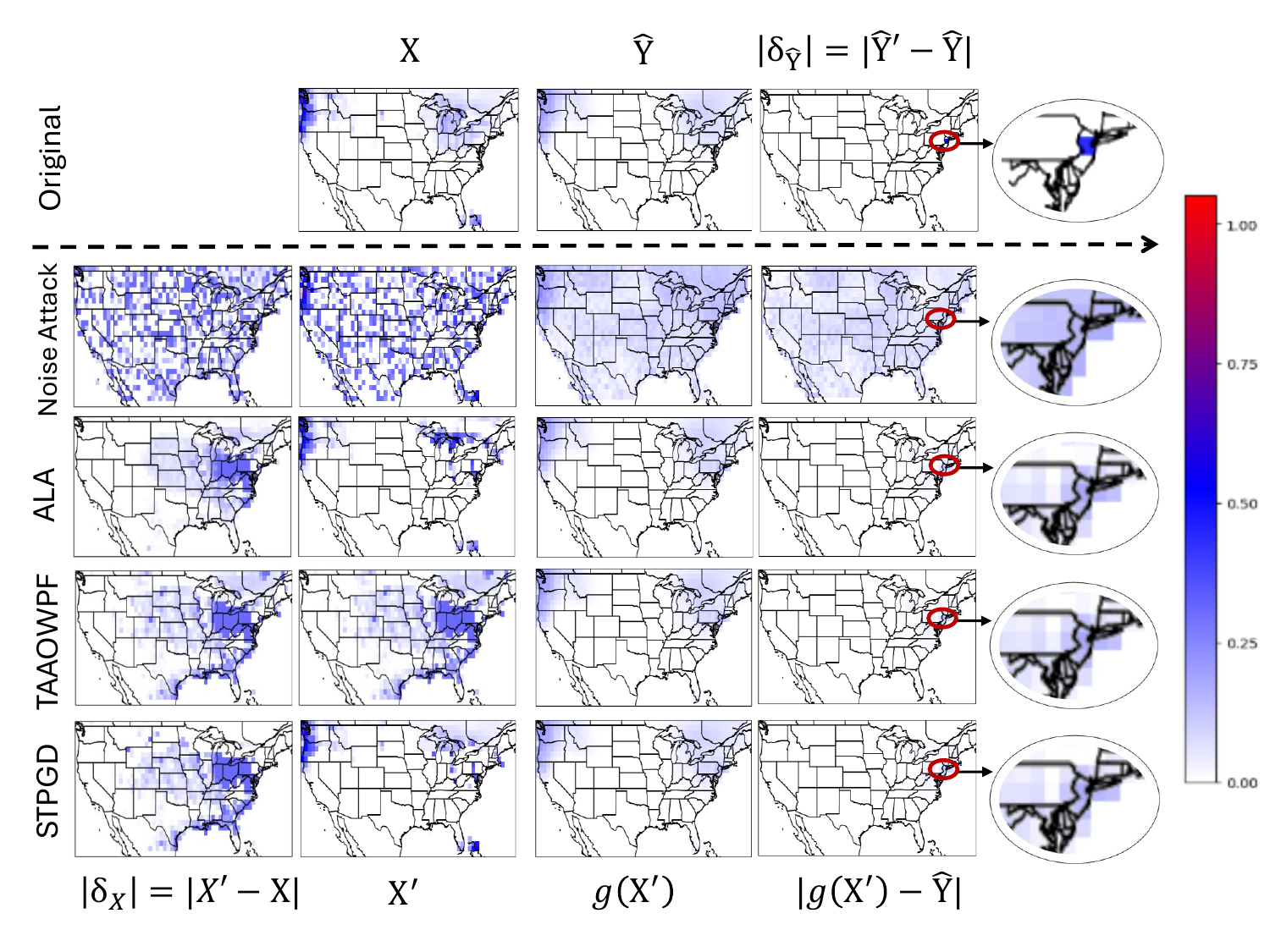}
    \caption{Examples of adversarial predictors $X'$ (rows 2-5, column 2) generated by different attack methods 
    when applied to the original predictor $X$ (row 1, column 2) of the {NLDAS} precipitation dataset~\cite{mitchell2004multi}. The first column shows the magnitude of perturbations (bounded by 2.5) introduced 
    by different attack methods. The third column shows the original forecast $\hat{Y}$ (row 1) and adversarial forecasts $g(X')$ (rows 2-5) produced by model $g$. The fourth column displays the absolute difference between original or adversarial.}
    \label{fig:baselines_apcpsfc}
\end{figure}

Figure~\ref{fig:baselines_results} compares the faithfulness and closeness of these methods when applied to the {NLDAS} precipitation dataset. 
Since the outputs of the adversarial attack methods depend on the perturbation bound $\epsilon$, 
we evaluate their performances at varying levels of $\epsilon$. 
As expected, the Noise Attack struggles to produce the adversarial target given its high in-AE and out-AE values, unlike existing learning-based methods such as {ALA}, {TAAOWPF}, and {STPGD}. Furthermore, increasing the perturbation bound $\epsilon$ leads to lower in-AE, suggesting that existing methods can generate adversarial predictor whose adversarial forecast is closer to the desired target by relaxing the perturbation bound. 
However, relaxing the perturbation bound leads to larger {out-AE}, i.e., undesired forecast errors outside the target region, and worse closeness, i.e., larger discrepancy between the original and adversarial predictors. This result underscores the \emph{\textbf{difficulty of achieving both faithfulness and closeness when applying existing methods}} to DLWF models. 

Figure~\ref{fig:baselines_apcpsfc} shows an example of adversarial predictors generated by each method for a targeted attack located near New York city. 
In terms of the {closeness} criterion, as expected, the Noise Attack, which perturbs the input data across the entire map, performs the worst compared to the learning-based adversarial attack methods ({TAAOWPF}, {STPGD}, and {ALA}). Among the learning-based methods, {ALA} produces perturbations with smaller magnitudes than others, while {STPGD} achieves slightly better {closeness} than {TAAOWPF} as it constrains the locations where perturbations are allowed to be applied. 
In terms of the {in-AE} and {out-AE} criteria, the learning-based methods exhibit relatively similar performance. Here, it can be clearly observed that the {out-AE} produced by the Noise Attack is far worse than that of the other methods. These observations are mostly consistent with the results shown in Figure~\ref{fig:baselines_results}, except that the {in-AE} achieved by the Noise Attack (as shown in the zoomed-in view of Figure~\ref{fig:baselines_apcpsfc}) appears comparable to those of the learning-based methods in this case.




\begin{figure}[t!]
    \centering
    \subfloat[\footnotesize \label{fig:moranI_tmp2m_a}]{
        \includegraphics[height=2.55cm]{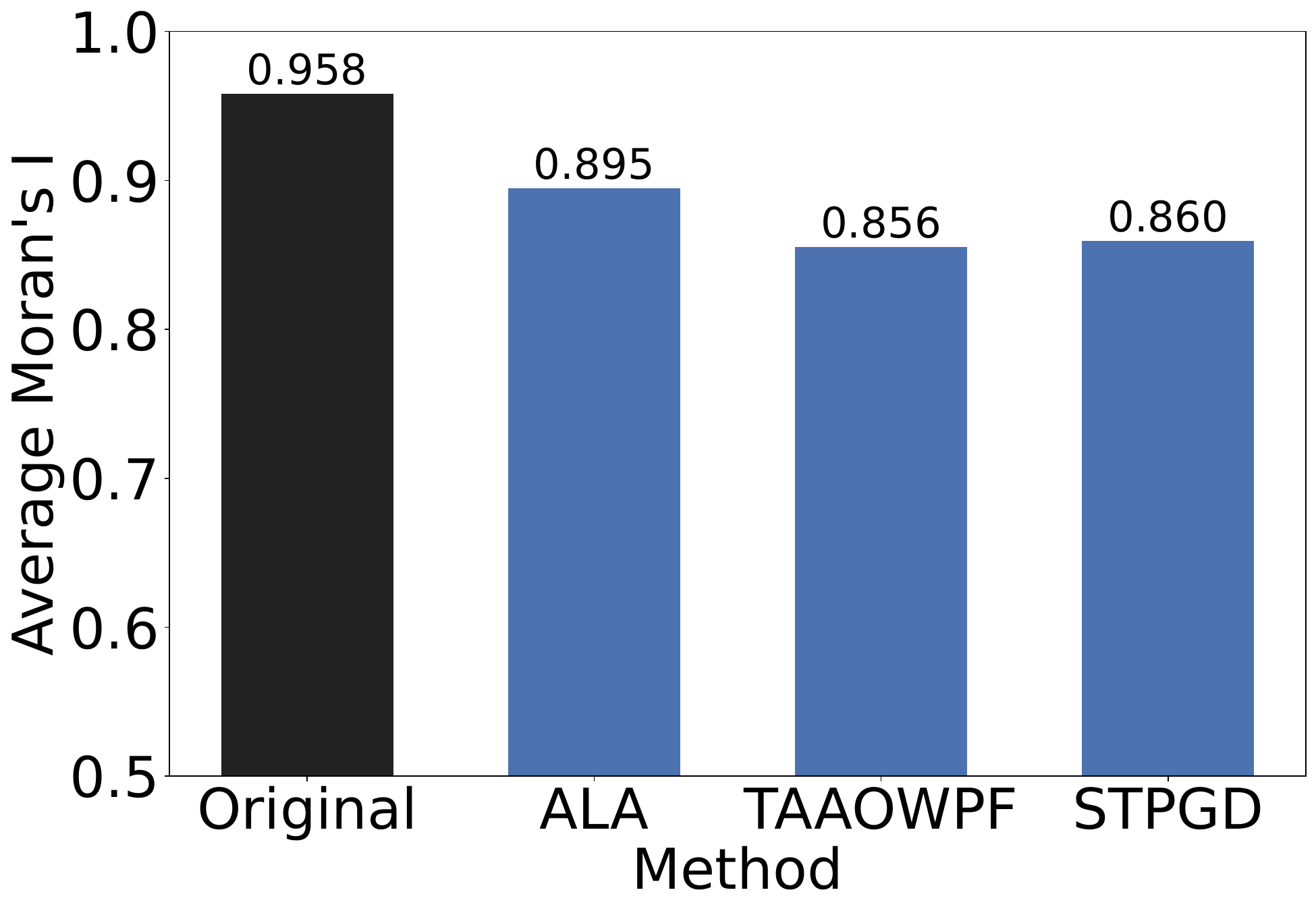}
    }
    \hfill
    \subfloat[\footnotesize \label{fig:moranI_tmp2m_b}]{
        \includegraphics[height=2.55cm]{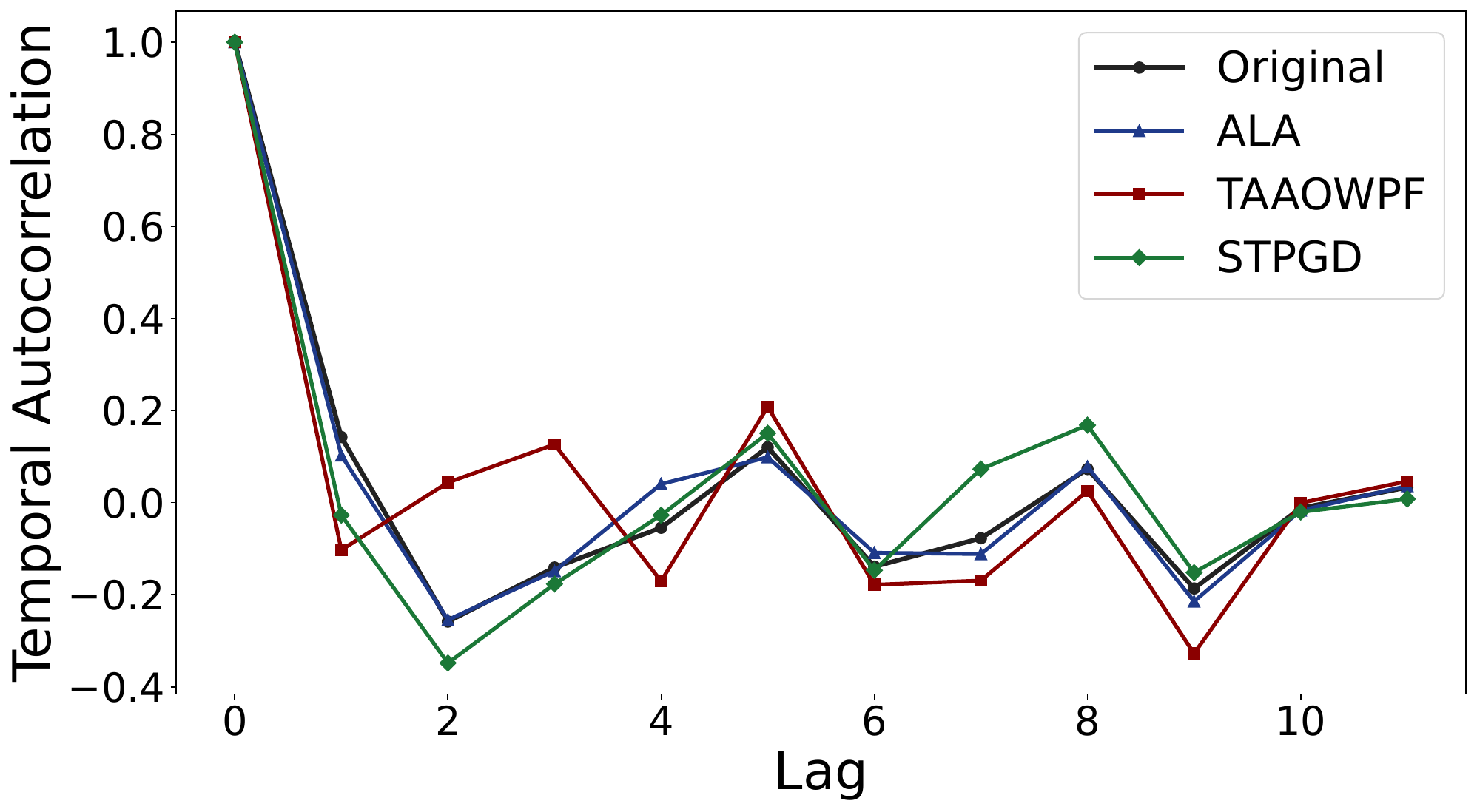}
    }
    \caption{Impact of existing adversarial attack methods on the geospatio-temporal realisticness of the {NLDAS} temperature dataset~\cite{mitchell2004multi}. Panel (a) reports the average Moran’s I over the period (daily records from January 2021 to March 2023) for the original input and adversarial inputs generated by different attack methods. Panel (b) shows temporal autocorrelation values, with the x-axis corresponding to different time lags.}
    \label{fig:MoranI_tmp2m}
\end{figure}


We also examine the geospatio-temporal realisticness of adversarial attacks for the {NLDAS} temperature dataset~\cite{mitchell2004multi}. 
We use  temperature instead of 
precipitation data since the former exhibits more significant spatial and temporal autocorrelations that must be preserved by the adversarial attack methods. 
Figure~\ref{fig:MoranI_tmp2m} shows the Moran's I and temporal autocorrelation for the adversarial predictors $\mathbf{X}'$ generated by the each method along with the autocorrelations present in the original predictor $\mathbf{X}$. 
The results suggest that the adversarial predictors generated by {TAAOWPF} and {STPGD} do not fully preserve the spatial and temporal autocorrelations present in the original predictor. For {ALA}, although its temporal autocorrelation remains relatively close to that of the original predictor, its Moran’s I still shows a clear discrepancy, indicating a distortion of spatial dependencies. This suggests that existing methods may not fully preserve the geospatio-temporal patterns inherent in the data..



\section{Proposed Framework: FABLE}
\label{sec:methods}


To address the limitations of existing methods, 
we propose a novel adversarial attack framework called \fable. 
Unlike conventional methods that directly perturb the original predictor $\mathbf{X}$, FABLE  decomposes $\mathbf{X}$ into its underlying components using a 3-D wavelet transform, enabling users to strategically adjust the perturbation magnitude across each component. 

\subsection{Components of FABLE}
As shown in Figure~\ref{fig:FABLE_architecture},  \fableb consists of 3 main components: (1) \textbf{Decomposition}, (2) \textbf{Perturbation}, and (3) \textbf{Reconstruction}. 

\subsubsection{Decomposition}
We employ a 3D Haar wavelet 
transform to decompose the original predictor $\mathbf{X}$ into its 
underlying frequency components along the temporal, longitudinal (column), and latitudinal (row) dimensions. The decomposition involves sequentially applying pairs of low-pass and high-pass filters across each dimension of $\mathbf{X} \in \mathbb{R}^{(\alpha+1) \times r \times c}$. The low-pass filter
extracts smooth, large-scale patterns corresponding to the \textit{low-frequency} \texttt{(L)} components of $\mathbf{X}$ while the high-pass filter
captures localized, fine-grained variations that represent its \textit{high-frequency} \texttt{(H)} components of $\mathbf{X}$. In the 3D case, the filters operate independently along each dimension of $\mathbf{X}$.  

\begin{figure}[t!]
    \centering
    \includegraphics[scale=0.255]{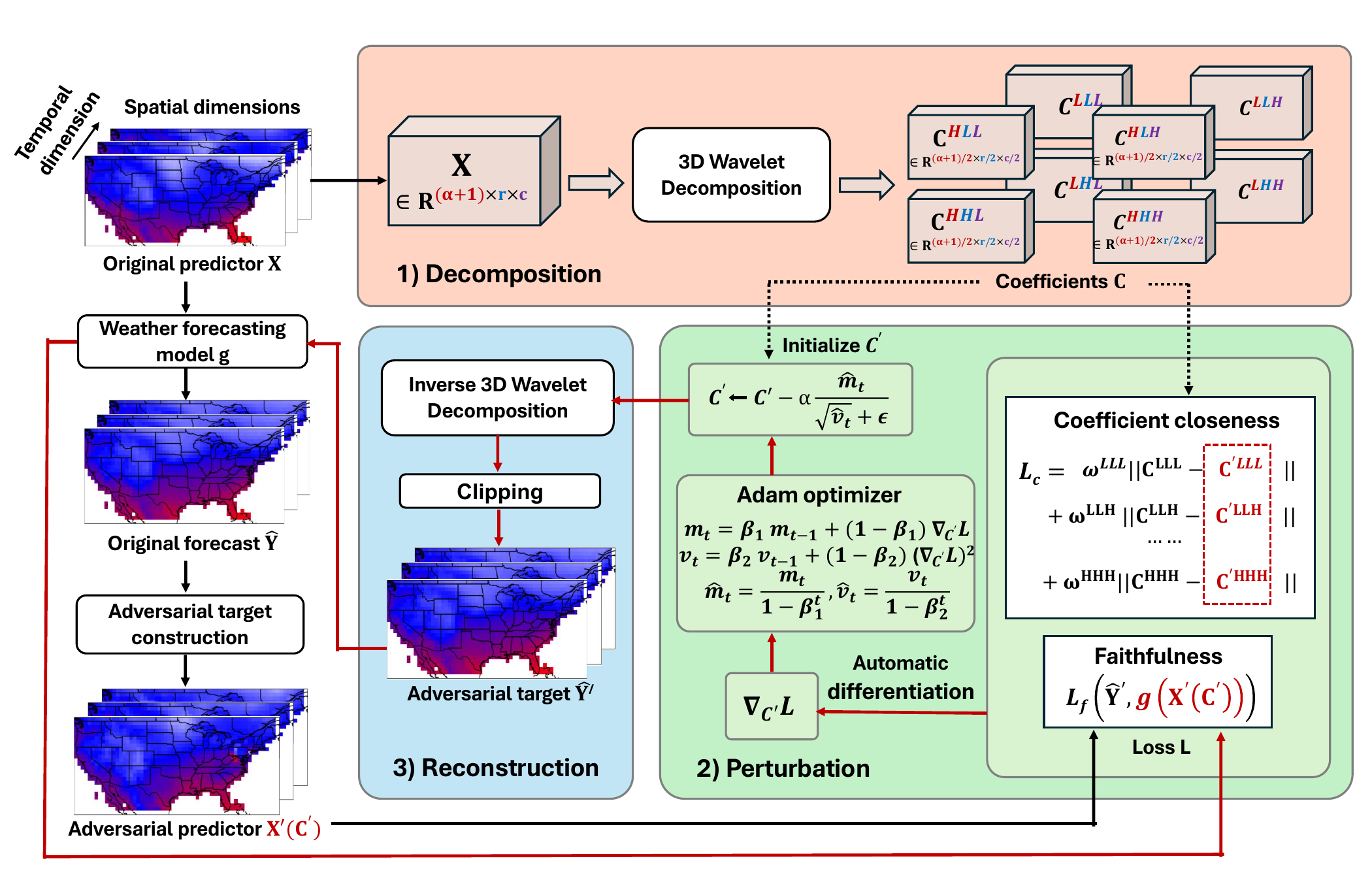}
    \caption{Proposed framework, consisting of (1) a \textbf{decomposition} step to extract 3D wavelet coefficients $\mathbf{C}$ from the original input $\mathbf{X}$, (2) a \textbf{perturbation} step to generate $\mathbf{C}'$ that minimizes the total loss $\mathcal{L}$, and (3) a \textbf{reconstruction} step using inverse 3D wavelet transform to generate the adversarial input $\mathbf{X}'$. 
    }
    \label{fig:FABLE_architecture}
\end{figure}

Let $\mathbf{f} = f_1f_2f_3$ be the combination of filters applied to $\mathbf{X}$ along its temporal, longitudinal, and latitudinal dimensions, where each $f_d \in \{\texttt{L}, \texttt{H}\}$ represents a low- or high-pass filter. For example, $\mathbf{f} = \texttt{LHH}$ denote the application of low-pass filter along the temporal dimension and high-pass filters along both spatial (longitudinal and latitudinal) dimensions.
Altogether, the 3D wavelet transform 
yields 8 frequency sub-bands---1 low-frequency component (\texttt{LLL}), 6 mixed-frequency components (\texttt{LLH}, \texttt{LHL}, \texttt{LHH}, \texttt{HLL}, \texttt{HLH}, \texttt{HHL}), and 1 high-frequency component (\texttt{HHH}). Let $\mathbf{C}^{\mathbf{f}} \in \mathbb{R}^{\frac{\alpha+1}{2} \times \frac{r}{2} \times \frac{c}{2}}$ be a tensor of wavelet coefficients associated with the sub-band $\mathbf{f}=f_1f_2f_3$, whose $(k_1,k_2,k_3)$-coefficient is computed as follows: 

\small
\begin{align}
    C^{\mathbf{f}}_{k_1, k_2, k_3} 
    &= \sum_{n_1=0}^{1} \sum_{n_2=0}^{1} \sum_{n_3=0}^{1} 
    \left[\frac{1}{\sqrt{2}} (-1)^{\xi_1 n_1}\right] \left[\frac{1}{\sqrt{2}} (-1)^{\xi_2 n_2}\right] \nonumber \\
    & \hspace{0.6cm}
    \times \left[\frac{1}{\sqrt{2}} (-1)^{\xi_3 n_3}\right] 
    \mathbf{X}_{2k_1 - n_1, 2k_2 - n_2, 2k_3 - n_3} \nonumber \\
    &= \frac{1}{\sqrt{8}} \sum_{n_1, n_2, n_3=0}^{1}  
    (-1)^{\boldsymbol{\xi}_{\mathbf{f}}^\top \mathbf{n}} 
    \mathbf{X}_{2k_1 - n_1, 2k_2 - n_2, 2k_3 - n_3},
    \label{eq:haar3d_decomposition}
\end{align}
\normalsize

\noindent where $k_1 \in \left\{1, 2, \dots, \frac{\alpha+1}{2} \right\}$,  
$k_2 \in \left\{1, 2, \dots, \frac{r}{2} \right\}$,  and
$k_3 \in \left\{1, 2, \dots, \frac{c}{2} \right\}$\footnote{By default, $\alpha + 1$, $r$, and $c$ are assumed to be even.} are the translation indices along the temporal, longitudinal, and latitudinal dimensions of $\mathbf{X}$, respectively, and $\boldsymbol{\xi}_{\mathbf{f}}= [\xi_1,\xi_2,\xi_3]$ is a binary vector whose $d$-th element is
$$
\xi_d =
\begin{cases}
0, & \text{if } f_d = L; \\
1, & \text{if } f_d = H,
\end{cases}
$$
and $\mathbf{n} = [n_1,n_2, n_3] \in \{0, 1\}^3$. 
By decomposing $\mathbf{X}$ into its distinct components using 3D discrete wavelet transform, this allows us to vary the perturbation magnitude across different components to generate an adversarial predictor. Detailed expressions and examples of the sub-band coefficients are provided in Appendix~\ref{appendix:wavelet_transform}.


\subsubsection{Perturbation} 
Given an adversarial target $\mathbf{\hat{Y}}'$, {FABLE} is designed to learn a set of perturbed wavelet coefficients, $\mathbf{C}'$, such that the reconstructed input $\mathbf{X}'(\mathbf{C}')$ 
induces a forecast $g(\mathbf{X}'(\mathbf{C}'))$ that closely aligns with the desired target $\mathbf{\hat{Y}}'$. Given a weather forecasting model $g$, the optimization objective for generating the adversarial input $\mathbf{X}'(\mathbf{C}')$ using {FABLE} is 
\begin{eqnarray}
\arg \min_{\mathbf{C}'} && \mathcal{L}_f\left(\mathbf{\hat{Y}}', g\left(\mathbf{X}'(\mathbf{C}')\right)\right) + \lambda \mathcal{L}_C\big( \mathbf{C},\mathbf{C}'; \mathbf{\omega}\big)\nonumber\\
\textrm{s.t.} && \|\mathbf{X} - \mathbf{X}'(\mathbf{C}')\|_{\infty} \leq \epsilon
\label{FABLE_loss_function}
\end{eqnarray}
{The first term, $\mathcal{L}_f\left(\mathbf{\hat{Y}}', g(\mathbf{X}')\right) = \frac{1}{\sqrt{\beta \times r \times c}}\|\mathbf{\hat{Y}}' -g(\mathbf{X}')\|_2$ ensures {faithfulness} of the reconstructed adversarial predictor $\mathbf{X}'$ by aligning its adversarial forecast $g(\mathbf{X}')$ to the target $\mathbf{\hat{Y}}'$. The second term, $\mathcal{L}_C\big( \mathbf{C},\mathbf{C}'; \omega\big) = \sum_{\mathbf{f}} \omega^{\mathbf{f}} \|\mathbf{C}^{\mathbf{f}} - \mathbf{C}'^{\mathbf{f}}\|_2$, ensures that the perturbation focuses more on high-frequency instead of low-frequency components by choosing the appropriate 
set of penalty weights, $\omega^{\mathbf{f}}\in \mathbb{R}$. 
The hyperparameter $\lambda$ balances these objectives, offering {FABLE} the flexibility to tailor adversarial predictors under varying conditions. Finally, the inequality constraint, $\|\mathbf{X} - \mathbf{X}'(\mathbf{C}')\|_{\infty} \leq \epsilon$, enforces {closeness} of $\mathbf{X}'$ to the original predictor $\mathbf{X}$, similar to other attack methods. We use the Adam optimizer~\cite{kingma2014adam} with clipping to solve this optimization problem. The perturbation constraint is enforced in each iteration by clipping 
$\mathbf{X}'$ to lie within the $\ell_\infty$ ball centered at $\mathbf{X}$ with radius $\epsilon$, 
i.e., $\mathbf{X}' \in [\mathbf{X}-\epsilon,\, \mathbf{X}+\epsilon]$. 

\subsubsection{Reconstruction} Given the perturbed wavelet coefficients $\mathbf{C}'^{\mathbf{f}} \in \mathbb{R}^{\frac{\alpha+1}{2} \times \frac{r}{2} \times \frac{c}{2}}$ and their corresponding Haar scaling and wavelet basis functions, the adversarial predictor $\mathbf{X}'(\mathbf{C}') \in \mathbb{R}^{(\alpha+1) \times r \times c}$ can be reconstructed using the inverse 3D Haar wavelet transform as follows:
\begin{eqnarray}
&&\mathbf{X}'_{2k_1 - n_1, 2k_2 - n_2, 2k_3 - n_3}(\mathbf{C}')\nonumber\\
&& = \ \frac{1}{\sqrt{8}} \sum_{f_1,f_2,f_3 \in \{L,H\}} (-1)^{\boldsymbol{\xi}_{\mathbf{f}}^T \mathbf{n}} \cdot C'^{\mathbf{f}}_{k_1, k_2, k_3}.
\label{Inverse_haar_wavelet_teansform}
\end{eqnarray}
where the summation is taken over all frequency sub-bands. 
Note that the reconstruction is performed on the perturbed wavelet coefficients $\mathbf{C}' = \{\mathbf{C}'^{\textit{LLL}}, \mathbf{C}'^{\textit{LLH}}, \cdots, \mathbf{C}'^{\textit{HHH}}\}$ prior to evaluating the loss in Equation~\eqref{FABLE_loss_function}. The perturbation-reconstruction steps are repeated until convergence (when the loss function does not change significantly or when the maximum number of epochs is reached). The pseudocode for the FABLE algorithm is presented in Algorithm~\ref{alg:fable}.

\begin{algorithm}[t]
\caption{FABLE}
\label{alg:fable}

\textbf{Parameters:} attack steps $N$; step size $\eta$; perturbation bound $\epsilon$;
regularization coefficient $\lambda$; penalty weights $\omega$.\\
\textbf{Input:} DLWF model $g$; original input $\mathbf{X}$; adversarial target $\hat{\mathbf{Y}}'$.\\
\textbf{Output:} adversarial input $\mathbf{X}'$.

\begin{algorithmic}[1]
\STATE $\mathbf{C} \leftarrow \mathrm{DWT}(\mathbf{X})$ \ \ 
\COMMENT{3D Haar wavelet transform}
\STATE $\mathbf{C}_0' \leftarrow \mathbf{C}$
\STATE $(\mathbf{m}_0,\mathbf{v}_0)\leftarrow(\mathbf{0},\mathbf{0})$
\COMMENT{Initialize parameters of ADAM}

\FOR{$t=1$ \TO $N$}
  \STATE {$\mathbf{X}_{t-1}' \leftarrow \mathrm{IDWT}(\mathbf{C}_{t-1}')$ \ \COMMENT{Inverse Haar transform}}
  \STATE $\mathbf{X}_{t-1}' \leftarrow \mathrm{Clip}(\mathbf{X}_{t-1}',\mathbf{X},\epsilon)$
  \STATE $\mathcal{L} \leftarrow 
  \mathcal{L}_f(g(\mathbf{X}_{t-1}'),\hat{\mathbf{Y}}')
  + \lambda\,\mathcal{L}_C(\mathbf{C},\mathbf{C}_{t-1}';\omega)$
  \FOR{each $\mathbf{f}\in\{\mathrm{LLL},\mathrm{LLH},\ldots,\mathrm{HHH}\}$}
    \STATE $\mathbf{g}_{t-1}^{\mathbf{f}} \leftarrow 
    \nabla_{\mathbf{C}_{t-1}'^{\mathbf{f}}}\mathcal{L}$
    \STATE $\mathbf{C}_t'^{\mathbf{f}} \leftarrow
    \mathrm{AdamUpdate}(\mathbf{C}_{t-1}'^{\mathbf{f}},
    \mathbf{g}_{t-1}^{\mathbf{f}},\eta,
    \mathbf{m}_{t-1},\mathbf{v}_{t-1})$
  \ENDFOR
\ENDFOR

\STATE {$\mathbf{X}_N' \leftarrow \mathrm{IDWT} (\mathbf{C}'_N)$ \ \COMMENT{Inverse Haar transform}}
\STATE \textbf{return} $\mathbf{X}_N'$
\end{algorithmic}
\end{algorithm}


\subsection{Perturbation Strategy for FABLE}

FABLE is designed to preserve both the proximity to the original predictor $\mathbf{X}$ and its geospatio-temporal realisticness  by modulating the perturbation magnitude so that the higher-frequency sub-bands are allowed to receive larger perturbations than lower-frequency ones. This is achieved by setting the penalty weights ($\omega^{\mathbf{f}}$) for higher frequency sub-bands to be lower than those for lower frequency sub-bands in the closeness loss term $\mathcal{L}_C$ defined in Equation \eqref{FABLE_loss_function}. 
To justify this strategy, the top-left panel of Figure~\ref{fig:perturbation_schema} shows $5$ different configurations of penalty weights $\omega^{\mathbf{f}}$. From configuration 1 to configuration 5, the design progressively favors perturbations toward the lower-frequency sub-bands by increasing the penalty weights on higher-frequency sub-bands, thereby further restricting their allowable perturbation magnitudes. Additionally, as the low-frequency (\texttt{LLL}) sub-band largely governs the overall time series structure, even minor perturbations to this component can induce substantial changes to the signal. To mitigate this effect, we exclude the low-frequency (\texttt{LLL}) sub-band from the closeness loss function in configurations 1 through 4 and include it only in configuration 5.


The remaining 5 panels in the figure summarize FABLE's performance under different configurations of $\omega^{\mathbf{f}}$. Each panel presents ten curves, each corresponding to one of ten selected U.S. cities used as target locations to construct adversarial targets for the same sample set. Each curve displays the average performance of the adversarial attack across different evaluation metrics. Observe that configuration 1 has the best closeness and geospatio-temporal realisticness among all configurations, but this improvement comes at the cost of poorer faithfulness. This suggests a strategy to improve closeness and geospatio-temporal realisticness is by increasing the perturbation magnitude on higher-frequency sub-bands.
This observation is further supported by  Theorems~\ref{upper_bound_comparison} and~\ref{thm:closeness} below. For brevity, the theorems are presented in the context of a 1-D signal. As the frequency coefficients are extracted independently from the spatial and temporal dimensions (see Equation~\eqref{eq:haar3d_decomposition}), the argument can be extended to each dimension in the 3-D case.

\begin{figure}[t!]
    \centering
    \includegraphics[scale=0.25]{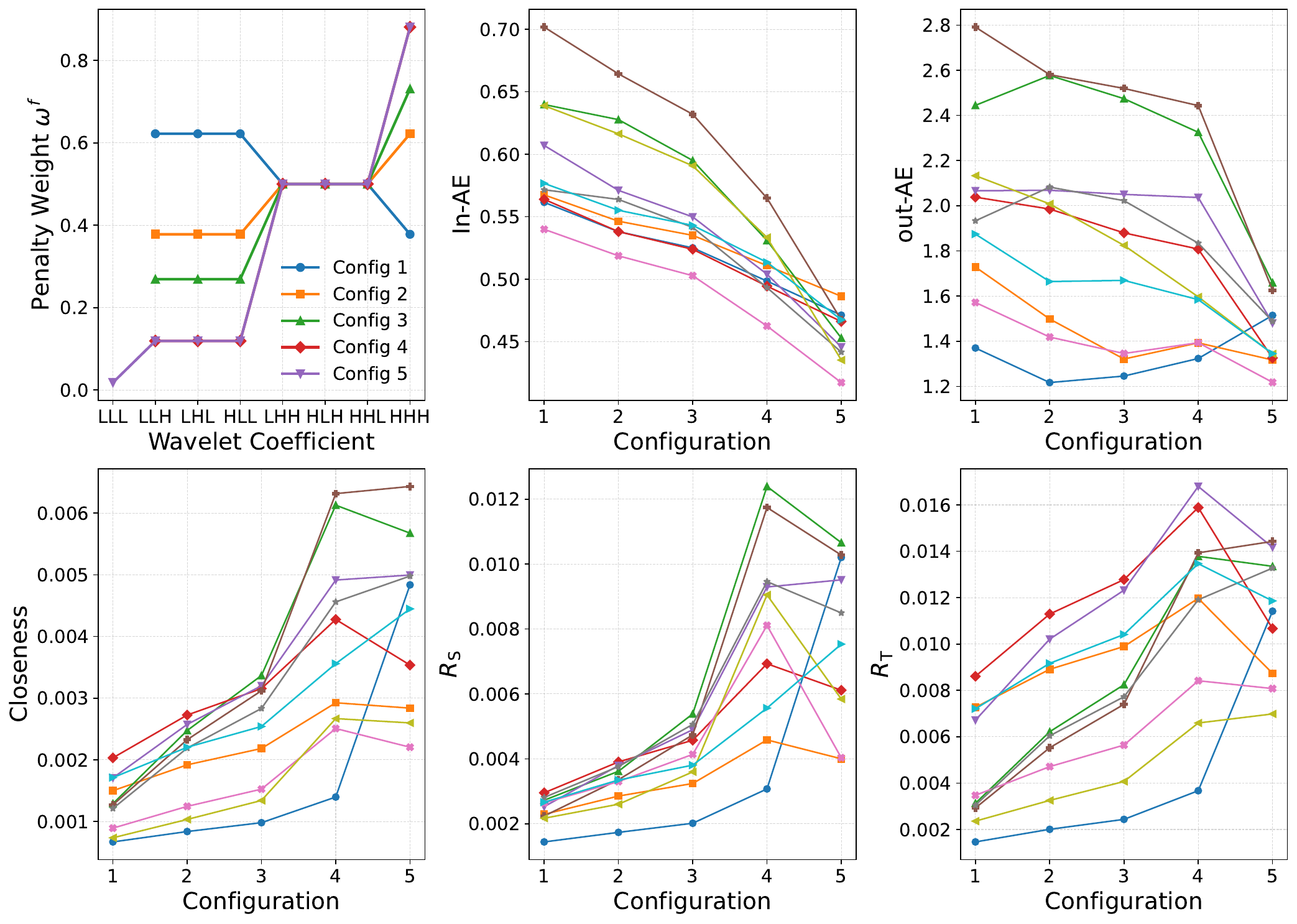}
    \caption{Effect of varying penalty weights for perturbations across different frequency sub-bands on samples from the NLDAS temperature dataset~\cite{mitchell2004multi}. From configuration $1$ to $5$, the perturbation penalty weight $\omega^f$ applied to higher-frequency sub-bands progressively increases, implying that the perturbation magnitude on higher-frequency sub-bands correspondingly decreases. For each metric, lower values indicate better performance.
    }
    \label{fig:perturbation_schema}
\end{figure}


Consider the Haar wavelet decomposition of a 1-D signal of length $T$:
$$
\mathbf{X}_{2k - n} = \frac{C^{\texttt{L}}_k}{\sqrt{2}} + \frac{(-1)^{1-n} C^{\texttt{H}}_k}{\sqrt{2}},
$$ 
where $n \in \{0,1\}$, $k \in \{1, \dots, T/2\}$, $\{C^{\texttt{L}}_k\}$ is the set of low-frequency components (approximation coefficients), and $\{C^{\texttt{H}}_k\}$ is the set of high-frequency components (detail coefficients). 
Let $\mathbf{X}'_A$ and $\mathbf{X}'_D$ be the signals obtained by perturbing only the low- and high-frequency components of $\mathbf{X}$, respectively:

\small
\begin{align*}
X’_A(2k-n)&=X(2k-n)+\frac{C_k'^L-C_k^L}{\sqrt{2}}= X(2k-n)+\frac{\delta_A(k)}{\sqrt{2}}\\
X'_{D}(2k-n)&=X(2k-n) + \frac{(-1)^{1-n}(C_k'^H-C_k^H)}{\sqrt{2}} \\
&= X(2k-n) + \frac{(-1)^{1-n}\delta_D(k)}{\sqrt{2}}
\end{align*}
\normalsize
where $\delta_A(k)$ and $\delta_D(k)$ are the perturbation magnitudes added to $C_k^L$ and $C_k^H$, respectively. 

Denote the autocorrelations at lag $l$ for the original and perturbed signals after standardization as $\rho(l)$, $\rho_A(l)$, and $\rho_D(l)$,  respectively, where $\rho(l)=\sum_{n}\sum_{k}X_{2k-n}X_{2k-n-l}$. The following theorem compares the change in autocorrelations at lag $l$ when perturbing only the detail (low-frequency) or the approximation (high-frequency) components of a 1-D signal.  

\begin{theorem}
\label{upper_bound_comparison}
Assuming
$
\sum_{k=1}^{T/2} |C^{\texttt{H}}_k| \leq \sum_{k=1}^{T/2} |C^{\texttt{L}}_k|$ and the perturbation magnitudes $\{\delta_A(k)\}$ and $\{\delta_D(k)\}$ are both upper bounded by some constant $\epsilon^*$, i.e., $\forall k: |\delta_A(k)|\leq \epsilon^*$ and $|\delta_D(k)|\leq \epsilon^*$, it follows that
$$
\sup_{\delta_D} \sum_{l=0}^{T-1} |\rho_D(l) - \rho(l)|
\leq
\sup_{\delta_A} \sum_{l=0}^{T-1} |\rho_A(l) - \rho(l)| 
.
$$
\end{theorem}

\begin{remark}
    The assumption $\sum_{k=1}^{T/2} |C^{\texttt{H}}_k|\leq \sum_{k=1}^{T/2} |C^{\texttt{L}}_k|$ generally holds in real-world settings as the approximation coefficients $\{|C^{\texttt{L}}_k|\}$ capture most of the signal's energy, while the detail coefficients $\{|C^{\texttt{H}}_k|\}$ primarily reflect its fine-scale fluctuations.
\end{remark}



\begin{remark}
If the condition $\sum_{k=1}^{T/2} |C^{\texttt{H}}_k| \leq \sum_{k=1}^{T/2} |C^{\texttt{L}}_k|$ holds, then it typically follows that $\|\frac{1}{\sqrt{2}}\boldsymbol{\delta}_D\|_2 \leq\|\frac{1}{\sqrt{2}}\boldsymbol{\delta}_A\|_2$. This is because the perturbation magnitudes $\delta_D$ and $\delta_A$ are generally proportional to the magnitudes of their corresponding wavelet coefficients, $C^H$ and $C^L$.
\end{remark}

Based on the preceding remarks, the theorem below compares the closeness of the original and perturbed signals.

\begin{theorem}
\label{thm:closeness}
Let $\mathbf{X}_A'$ and $\mathbf{X}_D'$ be the signals obtained by perturbing only the approximation coefficients and only the detail coefficients of $\mathbf{X}$, respectively.
If $\|\frac{1}{\sqrt{2}}\boldsymbol{\delta}_D\|_2 \leq \|\frac{1}{\sqrt{2}}\boldsymbol{\delta}_A\|_2$, then
$
\|\mathbf{X}_D' - \mathbf{X}\|_2 \leq \|\mathbf{X}_A' - \mathbf{X}\|_2.
$
\end{theorem}

\noindent The \textbf{proofs} for both theorems are provided in the Appendix.

\section{Experimental Setup}
\label{sec:Perfprmance Settings}






\subsection{Datasets and Preprocessing}
\label{sec:Datasets and Preprocessing}

We conducted our experiments on the following two benchmark meteorological datasets:
\subsubsection{North American Land Data Assimilation System (NLDAS)}
This dataset\footnote{\url{https://ldas.gsfc.nasa.gov/nldas}} contains the daily weather observations for 1,320 locations over a $1^\circ \times 1^\circ$ grid covering North America from 1979 to 2023. We consider 2 variables for  adversarial forecasting: \textit{2-meter air temperature} in Kelvin (\textbf{{NLDAS-TMP2M}}) and \textit{precipitation} in kg/m$^2$ (\textbf{{NLDAS-APCPSFC}}).
The data is split into 3 non-overlapping subsets, with the training set spanning 1979--2015, validation set from 2016--2019, and test set from 2020--2023. Following the setup in~\cite{lin2022conditional}, both the predictor and forecast windows are each set to a length of $12$ days. 

\subsubsection{ERA5 Reanalysis Data} This dataset\footnote{\url{https://www.ecmwf.int/en/forecasts/dataset/ecmwf-reanalysis-v5}} provides global hourly reanalysis weather data on a $5.625^\circ \times 5.625^\circ$ grid, covering 2,048 locations from 1979 to 2018. Adversarial inputs are generated for forecasting the \textit{2-meter air temperature} in Kelvin (\textbf{{ERA5-T2M}}) and \textit{total incident solar radiation} in $\text{Joules/m}^2$ (\textbf{{ERA5-TISR}}). 
To match the temporal resolution of the {FourCastNet}~\cite{pathak2022fourcastnet} model, we aggregate the  hourly data into six-hour intervals, averaging the values for {ERA5-T2M} and summing those for 
{ERA5-TISR}. 
We use the data from 1979–2016 for training, with 2017 for validation and 2018 for testing. Both the predictor and forecasting windows are set to a length of $12$ time steps, corresponding to 3-day window.


A summary statistics of the datasets is shown in Table~\ref{tab:data_statistics_compact_vertical}. For both datasets, we remove the data with missing values and standardize the remaining observations by centering them using the global mean and standard deviation, computed over all locations and time steps. A logarithmic transformation is also applied to the {NLDAS-APCPSFC} data. 
For each dataset, we train a DLWF model on the training set, tune its hyperparameters based on performance on the validation set, and perform adversarial attacks exclusively on the test set.

\begin{table}[t]
\centering
\setlength{\tabcolsep}{3.5pt}
\renewcommand{\arraystretch}{1.05}
\begin{tabular}{lcccc}
\toprule
Statistic
& APCPSFC
& TMP2M
& T2M
& TISR \\
& (kg/m$^2$) & ($^\circ$K) & ($^\circ$K) & (J/m$^2$) \\
\midrule
Mean
& 0.0832
& 283.9300
& 278.2700
& $6.4406 \times 10^{6}$ \\
Standard Deviation
& 0.2071
& 38.3300
& 21.0500
& $7.7218 \times 10^{6}$ \\
Minimum
& 0
& 232.5900
& 193.6600
& 0 \\
Maximum
& 10.1873
& 314.5700
& 317.8200
& $2.7871 \times 10^{7}$ \\
25th Percentile
& 0
& 275.4700
& 268.8100
& 0 \\
50th Percentile
& 0.0050
& 285.3900
& 283.2100
& $2.8250 \times 10^{6}$ \\
75th Percentile
& 0.0621
& 293.4100
& 295.9000
& $1.1292 \times 10^7$ \\
90th Percentile
& 0.2388
& 298.9900
& 299.6300
& $1.9387 \times 10^{7}$ \\
95th Percentile
& 0.4357
& 301.2700
& 300.5300
& $2.2900 \times 10^{7}$ \\
99th Percentile
& 1.0476
& 304.4300
& 302.3900
& $2.5913 \times 10^{7}$ \\
\bottomrule
\end{tabular}
\caption{Summary statistics of the weather variables
computed from valid data over all locations and time steps.}
\label{tab:data_statistics_compact_vertical}
\end{table}


\subsection{DLWF Models} 



We consider two DLWF models for our experiments: \textit{CLCRN}~\cite{lin2022conditional} and \textit{FourCastNet}~\cite{pathak2022fourcastnet}. The {CLCRN} model was trained to predict the 
{precipitation (apcpsfc)} and {2-meter air temperature (tmp2m)} variables in the {NLDAS} dataset, while the {FourCastNet} model was trained to forecast the two {ERA5} variables: {2-meter air temperature (t2m)}, and {total incident solar radiation (tisr)}. While {CLCRN} supports multi-step forecasting, {FourCastNet} was originally designed to predict multiple variables at a single time step. 
Specifically, the original {FourCastNet} takes the full multi-variable atmospheric state at the current time as input and predicts the same set of variables at the next time step. We adapt {FourCastNet} by providing a sequence of past observations of the target variable as input and train the model to directly predict that target variable over the next 12 time steps. All the models were trained using their default hyperparameter configurations\footnote{The original codes of {CLCRN} and {FourCastNet} models are publicly available at \url{https://github.com/EDAPINENUT/CLCRN} and \url{https://github.com/NVlabs/FourCastNet}, respectively.}, with necessary adjustments limited only to input/output formatting and task-specific adaptations using the preprocessed training and validation data described in Section~\ref{sec:Datasets and Preprocessing}. The forecast performance of these models on their respective test sets are reported 
in Table~\ref{tab:clean_performance}. Note that our method is model-agnostic, and thus, is applicable to any DLWF models.


\begin{table}[!t]
\centering
\small
\begin{tabular}{lccl}
\toprule
Model & Dataset & Target Variable & MAE \\
\midrule
CLCRN & NLDAS & APCPSFC & $0.0807$ kg/m$^2$ \\
CLCRN & NLDAS & TMP2M & $3.3898^\circ$K \\
FourCastNet & ERA5 & T2M & $1.8364^\circ$K \\
FourCastNet & ERA5 & TISR & $1.1755 \times 10^5$ J/m$^2$\\
\bottomrule
\end{tabular}
\caption{Mean Absolute Error (MAE) of the trained DLWF models on the test set for each target variable.
}
\label{tab:clean_performance}
\end{table}






\subsection{Adversarial Attack Methods}

We compare the performance of {FABLE} against 5 baselines. As some of these methods were originally developed for untargeted attacks or classification tasks, they must be adapted to a targeted attack setting for regression problems to minimize the loss $\mathcal{L}\left(g(\mathbf{X}'),\mathbf{\hat{Y}}'\right) = \|g(\mathbf{X'}) - \mathbf{\hat{Y}'}\|$. 

\subsubsection{Noise Attack~\cite{heinrich2024targeted}} This method generates multiple candidates for $\mathbf{X'}$ by adding random Gaussian noise to $\mathbf{X}$ and selects the candidate that minimizes the loss $\mathcal{L}(g(\mathbf{X}'),\mathbf{\hat{Y}}')$.

\subsubsection{FGSM~\cite{goodfellow2014explaining}} This method performs a one-step gradient update on $\mathbf{X}$ 
to find $\mathbf{X'}$ 
such that $g(\mathbf{X'}) \approx \mathbf{\hat{Y}'}$. The update formula for FGSM is given by 
$
\mathbf{X}^{'}=\operatorname{Clip}_{\mathbf{X},\epsilon}\{\mathbf{X}-\alpha \operatorname{sign} (\nabla_{\mathbf{X}}\mathcal{L}(g(\mathbf{X}),\mathbf{\hat{Y}}')\},
$
where $\alpha$ controls the perturbation magnitude, and the clipping function 
$\operatorname{Clip}_{\mathbf{X},\epsilon}(\cdot)$ is used to ensure that $\mathbf{X}^{'(i)}$ remains within the $\epsilon$-ball centered at $\mathbf{X}$. 

\subsubsection{ALA~\cite{ruan2023vulnerability}} 
This method performs untarget adversarial attack using multi-step gradient updates with Adam optimizer~\cite{kingma2014adam} to manipulate 
renewable energy forecasts. 
We adapt {ALA} 
to minimize the loss $\mathcal{L}(g(\mathbf{X}'),\mathbf{\hat{Y}}')$. At each epoch 
$i$,
$
\mathbf{X}^{'(i+1)} = \operatorname{Clip}_{\mathbf{X},\epsilon}\{\mathbf{X}^{'(i)} - \alpha \cdot \frac{\hat{m}_i}{\sqrt{\hat{v}_i} + \epsilon}\}
$, where $\mathbf{X^{'(0)}}=\mathbf{X}$. Here, $\hat{m}_i = \frac{\beta_1 m_{i-1} + (1 - \beta_1) \nabla \ell_i}{1 - \beta_1^i}$ and $\hat{v}_i = \frac{\beta_2 v_{i-1} + (1 - \beta_2) \nabla \ell_i^2}{1 - \beta_2^i}$ and 
$\nabla \ell_i$ is the gradient of the loss $\mathcal{L}(g(\mathbf{X}'^{(i)}),\mathbf{\hat{Y}}')$.

\subsubsection{TAAOWPF~\cite{heinrich2024targeted}}
This method applies projected gradient descent (PGD), which is an iterative extension of FGSM, to manipulate multi-step forecasts of total wind power production. 
We adapt {TAAOWPF} as a PGD method that perturbs $\mathbf{X}$ to minimize the loss $\mathcal{L}(g(\mathbf{X}'),\mathbf{\hat{Y}}')$. At each iterative step $i$,: 
$
\mathbf{X}^{'(i+1)}=\operatorname{Clip}_{\mathbf{X}, \epsilon}\{\mathbf{X}^{'(i)}-\alpha \operatorname{sign} (\nabla_{\mathbf{X}^{'(i)}}\mathcal{L}((g(\mathbf{X}'^{(i)}),\mathbf{\hat{Y}}')))$.

\subsubsection{STPGD~\cite{liu2022practical}} 
This method performs untargeted adversarial attacks on traffic prediction models by perturbing node features in a traffic network. Instead of perturbing all the nodes, it identifies a set of high-saliency victim nodes using time-dependent node saliency (TDNS) scores and applies PGD perturbations only to those selected nodes. To adapt
{STPGD} to our problem, we first calculate the saliency score for each location $(i,j)$ as $\|\operatorname{ReLU}(\nabla_{\mathbf{X}_{ij}}\mathcal{L}(g(\mathbf{X}_{ij}),\mathbf{\hat{Y}'}_{ij}))\|_2$. The top-$k$ locations with highest saliency scores are chosen as the victim locations. Perturbation is then performed at the chosen locations using PGD:
$
\mathbf{X}^{'(i+1)}=\operatorname{Clip}_{\mathbf{X}, \epsilon}\{\mathbf{X}^{'(i)}-\alpha \operatorname{sign} (\nabla_{\mathbf{X}^{'(i)}}\mathcal{L}((g(\mathbf{X}'^{(i)}),\mathbf{\hat{Y}}'))\cdot \mathbf{V}\},
$
where $\mathbf{V} \in \{0,1\}^{(r \times s)\times(r \times s)}$ is a diagonal matrix indicating whether $(i,j)$ is a victim location.



\subsection{Adversarial Target Construction}
\label{sec:Construct Adversarial Target}


Given the original predictor $\mathbf{X} \in \mathbb{R}^{(\alpha+1)\times r \times c}$ and a DLWF model $g$, let $\mathbf{\hat{Y}}=g(\mathbf{X}) \in \mathbb{R}^{\beta \times r \times c}$ be the original forecast.
We construct the adversarial target $\mathbf{\hat{Y}}'$ from the original forecast $\mathbf{\hat{Y}}$ as follows. First, we select a target location $(i_c,j_c) \in \{1,...,r\}\times \{1,...,c\}$ and a time step $\tau_c \in [5,6]$ for applying the perturbation. The adversarial target value is then randomly sampled from a predefined range as shown in Table~\ref{table:adversarial_target_construction}. If the original forecast value is extreme, the target is set to be non-extreme; otherwise, it is converted to an extreme value.

To ensure the adversarial target is geospatio-temporally smooth, we also perturb its adjacent time steps and neighboring locations as follows:
\begin{align}
\hat{Y}'_{\tau i_cj_c} &= 
    \hat{Y}_{\tau i_cj_c}+\delta_{\hat{Y}_{\tau_c i_cj_c}} e^{-\frac{(\tau -\tau_c)^2}{\sigma_{\tau}^2}} \nonumber\\
\hat{Y}'_{\tau i j} &= 
    \hat{Y}_{\tau ij}+\delta_{\hat{Y}_{\tau i_c j_c}}e^{-\frac{d_{(i_c j_c),(ij)}}{\sigma_d}} \nonumber
\end{align}
where 
$\delta_{\hat{Y}_{\tau i_c j_c}}=\hat{Y}'_{\tau i_c j_c}-\hat{Y}_{\tau i_c j_c}$, 
$\sigma_{\tau}$ and $\sigma_d$ control the smoothness of the perturbation along the temporal and spatial dimensions, respectively, 
and $d_{(i_c,j_c),(i,j)} = \sqrt{(i-i_c)^2+(j-j_c)^2}$ measures the distance between two locations. 

\begin{table*}[t]
\centering
\caption{Configuration settings for adversarial target construction. 
\textbf{(i)} For {NLDAS} data, the perturbed location $(i,j)$ is selected from ten major U.S. cities. 
Perturbations are applied only at the selected cities. 
\textbf{(ii)} For {ERA5} data, the perturbed location $(i,j)$ is randomly selected. Perturbations are also applied to its neighbors, defined as its eight adjacent grid cells. 
For all datasets, adversarial targets are clipped to remain within the specified ranges, where $Y_{\cdot\%}$ is the percentile value of $Y$. 
}
\label{table:adversarial_target_construction}
\begin{tabular}{llcll}
\toprule
Dataset & Variable
&Extreme value
& Range of adversarial target 
& Range of adversarial target  \\
& (unit) & threshold ($Y^*$) & $\hat{Y}'_{\tau i j}$ when $\hat{Y}_{\tau ij}< Y^*$ & $\hat{Y}'_{\tau i j}$ when $\hat{Y}_{\tau ij}\ge Y^*$ \\
\midrule

\textit{NLDAS} & {APCPSFC (kg/m$^2$)} 
& $Y_{95\%}$
& $ [Y_{95\%}, Y_{max}]$
& $ [Y_{min}, Y_{25\%}]$
\\

\textit{NLDAS} & {TMP2M ($^\circ$K)} 
& $Y_{50\%}$
& 
$[\hat{Y}_{\tau i j} + 9, \hat{Y}_{\tau i j} + 10]$

& 
$[\hat{Y}_{\tau i j} - 10, \hat{Y}_{\tau i j} - 9]$
\\

\textit{ERA5} & {T2M ($^\circ$K)} 
& $Y_{50\%}$ 
& 
$[\hat{Y}_{\tau i j} + 9, \hat{Y}_{\tau i j} + 10]$
& 
$[\hat{Y}_{\tau i j} - 10, \hat{Y}_{\tau i j} - 9]$
\\

\textit{ERA5} & {TISR (J/m$^2$)} 
& $Y_{50\%}$ 
& 
$[\hat{Y}_{\tau i j} + 3.10 \times 10^6, \hat{Y}_{\tau i j} + 3.86 \times 10^6]$
& 
$ [ \hat{Y}_{\tau i j} - 3.86 \times 10^6, \hat{Y}_{\tau i j} - 3.1 \times 10^6]$
\\
\bottomrule
\end{tabular}
\end{table*}

The adversarial targets are constructed on the test set of each dataset. For the {NLDAS} datasets, each test set contains 96 samples. We select 10 major cities as targets for the adversarial attack. For each test sample, adversarial targets are generated for all 10 cities, resulting in 960 samples per run ($96 \times 10$). This procedure is repeated five times with different target values, producing a total of 4,800 test samples for each {NLDAS} dataset. For the {ERA5} datasets, each test set contains 60 samples. Similarly, we generate 10 adversarial targets per sample and repeat the process five times, yielding a total of 3,000 test samples for each {ERA5} dataset. 

\subsection{Evaluation Metrics}

We evaluate the performance of the different methods using the same criteria defined in Section~\ref{Preliminaries}. For \textit{\textbf{closeness}}, we measure the $\ell_1$ distance between the original and adversarial inputs, as shown in Equation~\eqref{eqn:closeness}. For \textit{\textbf{faithfulness}}, we compute both the in-target and out-target absolute errors of the adversarial forecasts, as defined in Equations~\eqref{eqn:in-ae} and \eqref{eqn:out-ae}. Finally, \textit{\textbf{geospatio-temporal realisticness}} is evaluated based on the differences in spatial and temporal autocorrelations ($R_S$ and $R_T$) between the original and adversarial inputs, as defined in Equations~\eqref{eq:spatial_smoothness} and \eqref{eq:temporal_smoothness}. For each criterion, we first compute the mean value for each run and then report the average and standard deviation across all five runs. 

\subsection{Hyperparameter Selection}


Following the approach in~\cite{heinrich2024targeted}, the learning rates for all attack methods are set to $\frac{2\epsilon}{N}$, where $\epsilon$ denotes the clipping threshold and $N$ is the number of attack iterations. As shown in Figure~\ref{fig:baselines_results}, increasing $\epsilon$ consistently improves {In-AE} for the baseline attack methods. However, both {Out-AE} and {closeness} deteriorate as $\epsilon$ grows. 
Since our objective is to evaluate the effectiveness of localized targeted attacks, we set $\epsilon = 2.5$ for all methods (including {FABLE}) to achieve a reasonable balance between minimizing {In-AE} and maintaining acceptable levels of {Out-AE} and {closeness}. 
We set the number of iterations $N$ sufficiently large (between $1000$ to $2000$) to ensure convergence of the learning process, stopping when the loss stabilizes to four decimal places. 

For {STPGD}, the number of salient locations is set to $990$ for the {NLDAS} dataset and $1536$ for the {ERA5} dataset to balance attack effectiveness and computational cost. For FABLE, the regularization strength is set to $\lambda = 10^{-6}$. Based on the analysis in Figure~\ref{fig:perturbation_schema}, we assign weights of $0.8$ to the \texttt{LLH}, \texttt{LHL}, and \texttt{HLL} sub-bands; $0.5$ to the \texttt{LHH}, \texttt{HLH}, and \texttt{HHL} sub-bands; and $0.2$ to the \texttt{HHH} sub-band across all datasets. These weights are chosen to balance faithfulness with closeness and geospatio-temporal realisticness criteria.

\section{Experimental Results}
\label{sec:Perfprmance Results}

\begin{table*}[t!]
\footnotesize
\centering
\setlength{\tabcolsep}{1.3pt} 
\caption{Performance comparison on different datasets with adapted baselines. 
For each metric, \textcolor{deepred}{red} entries indicate the best performance, and \textcolor{deepblue}{blue} entries indicate the second-best performance. Results are reported based on standardized data. On NLDAS datasets, {CLCRN}~\cite{lin2022conditional} is the attacked DLWF model. On ERA5 datasets,  {FourCastNet}~\cite{pathak2022fourcastnet} is the attacked DLWF model. 
}
\label{CLCRN-NLDAS}
\begin{tabular}{l|ccccc|ccccc}
\hline
& \multicolumn{5}{c|}{NLDAS-TMP2M} 
& \multicolumn{5}{c}{ERA5-T2M}  \\
Method 
& in-AE $\downarrow$
& out-AE $\downarrow$ 
& Closeness $\downarrow$ 
& $R_S$ $\downarrow$ 
& $R_T$ $\downarrow$

& in-AE $\downarrow$
& out-AE $\downarrow$ 
& Closeness $\downarrow$ 
& $R_S$ $\downarrow$ 
& $R_T$ $\downarrow$\\

\hline
\textbf{FABLE} 
& \begin{tabular}[c]{@{}c@{}}{}\textcolor{blue}{0.3745}\\($\pm$0.0006)\end{tabular} 
& \begin{tabular}[c]{@{}c@{}}{}\textcolor{red}{1.3139}\\($\pm$0.0042)\end{tabular}  
& \begin{tabular}[c]{@{}c@{}}{}\textcolor{red}{0.0140}\\($\pm$0.0000)\end{tabular}   
& \begin{tabular}[c]{@{}c@{}}{}0.0624\\($\pm$0.0004)\end{tabular}   
& \begin{tabular}[c]{@{}c@{}}{}\textcolor{red}{0.0286}\\($\pm$0.0001)\end{tabular} 

& \begin{tabular}[c]{@{}c@{}}{}3.0466\\($\pm$0.0374)\end{tabular} 
& \begin{tabular}[c]{@{}c@{}}{}\textcolor{blue}{21.5905}\\($\pm$0.1184)\end{tabular} 
& \begin{tabular}[c]{@{}c@{}}{}\textcolor{red}{0.0087}\\($\pm$0.0000)\end{tabular} 
& \begin{tabular}[c]{@{}c@{}}{}\textcolor{red}{0.0018}\\($\pm$0.0000)\end{tabular}  
& \begin{tabular}[c]{@{}c@{}}{}\textcolor{red}{0.0289}\\($\pm$0.0002)\end{tabular} \\

ALA 
& \begin{tabular}[c]{@{}c@{}}{}\textcolor{red}{0.3631}\\($\pm$0.0006)\end{tabular} 
& \begin{tabular}[c]{@{}c@{}}{}\textcolor{blue}{2.3987}\\($\pm$0.0240)\end{tabular} 
& \begin{tabular}[c]{@{}c@{}}{}\textcolor{blue}{0.0149}\\($\pm$0.0000)\end{tabular} 
& \begin{tabular}[c]{@{}c@{}}{}\textcolor{red}{0.0462}\\($\pm$0.0002)\end{tabular} 
& \begin{tabular}[c]{@{}c@{}}{}\textcolor{blue}{0.0347}\\($\pm$0.0001)\end{tabular}

& \begin{tabular}[c]{@{}c@{}}{}\textcolor{red}{0.6319}\\($\pm$0.0363)\end{tabular} 
& \begin{tabular}[c]{@{}c@{}}{}\textcolor{red}{14.4565}\\($\pm$2.7890)\end{tabular} 
& \begin{tabular}[c]{@{}c@{}}{}\textcolor{blue}{0.0207}\\($\pm$0.0012)\end{tabular} 
& \begin{tabular}[c]{@{}c@{}}{}\textcolor{blue}{0.0019}\\($\pm$0.0001)\end{tabular} 
& \begin{tabular}[c]{@{}c@{}}{}0.0738\\($\pm$0.0032)\end{tabular} \\

TAAOWPF
& \begin{tabular}[c]{@{}c@{}}{}0.4114\\($\pm$0.0006)\end{tabular} 
& \begin{tabular}[c]{@{}c@{}}{}9.7471\\($\pm$0.0123)\end{tabular} 
& \begin{tabular}[c]{@{}c@{}}{}0.0348\\($\pm$0.0001)\end{tabular} 
& \begin{tabular}[c]{@{}c@{}}{}0.0661\\($\pm$0.0002)\end{tabular} 
& \begin{tabular}[c]{@{}c@{}}{}0.1043\\($\pm$0.0002)\end{tabular} 

& \begin{tabular}[c]{@{}c@{}}{}\textcolor{blue}{0.6881}\\($\pm$0.0173)\end{tabular} 
& \begin{tabular}[c]{@{}c@{}}{}36.6020\\($\pm$0.0182)\end{tabular} 
& \begin{tabular}[c]{@{}c@{}}{}0.0210\\($\pm$0.0000)\end{tabular} 
& \begin{tabular}[c]{@{}c@{}}{}\textcolor{blue}{0.0019}\\($\pm$0.0000)\end{tabular} 
& \begin{tabular}[c]{@{}c@{}}{}0.0790\\($\pm$0.0001)\end{tabular} \\

STPGD 
& \begin{tabular}[c]{@{}c@{}}{}0.4238\\($\pm$0.0041)\end{tabular} 
& \begin{tabular}[c]{@{}c@{}}{}8.0782\\($\pm$0.1243)\end{tabular} 
& \begin{tabular}[c]{@{}c@{}}{}0.0273\\($\pm$0.0008)\end{tabular} 
& \begin{tabular}[c]{@{}c@{}}{}\textcolor{blue}{0.0616}\\($\pm$0.0034)\end{tabular} 
& \begin{tabular}[c]{@{}c@{}}{}0.0877\\($\pm$0.0007)\end{tabular} 

& \begin{tabular}[c]{@{}c@{}}{}6.0915\\($\pm$0.1642)\end{tabular} 
& \begin{tabular}[c]{@{}c@{}}{}52.3506\\($\pm$0.5439)\end{tabular} 
& \begin{tabular}[c]{@{}c@{}}{}0.0226\\($\pm$0.0004)\end{tabular} 
& \begin{tabular}[c]{@{}c@{}}{}0.0023\\($\pm$0.0001)\end{tabular} 
& \begin{tabular}[c]{@{}c@{}}{}0.0597\\($\pm$0.0007)\end{tabular} \\

Noise Attack 
& \begin{tabular}[c]{@{}c@{}}{}2.2152\\($\pm$0.0058)\end{tabular} 
& \begin{tabular}[c]{@{}c@{}}{}287.9209\\($\pm$0.1526)\end{tabular} 
& \begin{tabular}[c]{@{}c@{}}{}0.0622\\($\pm$0.0000)\end{tabular} 
& \begin{tabular}[c]{@{}c@{}}{}0.1208\\($\pm$0.0001)\end{tabular} 
& \begin{tabular}[c]{@{}c@{}}{}0.1536\\($\pm$0.0000)\end{tabular} 

& \begin{tabular}[c]{@{}c@{}}{}37.6243\\($\pm$0.2790)\end{tabular} 
& \begin{tabular}[c]{@{}c@{}}{}6237.1805\\($\pm$2.1947)\end{tabular} 
& \begin{tabular}[c]{@{}c@{}}{}0.1709\\($\pm$0.0001)\end{tabular} 
& \begin{tabular}[c]{@{}c@{}}{}0.0435\\($\pm$0.0000)\end{tabular} 
& \begin{tabular}[c]{@{}c@{}}{}0.2253\\($\pm$0.0001)\end{tabular} \\

FGSM 
& \begin{tabular}[c]{@{}c@{}}{}0.7506\\($\pm$0.0011)\end{tabular} 
& \begin{tabular}[c]{@{}c@{}}{}1496.2828\\($\pm$0.2437)\end{tabular} 
& \begin{tabular}[c]{@{}c@{}}{}0.1000\\($\pm$0.0000)\end{tabular} 
& \begin{tabular}[c]{@{}c@{}}{}0.1117\\($\pm$0.0000)\end{tabular} 
& \begin{tabular}[c]{@{}c@{}}{}0.1284\\($\pm$0.0000)\end{tabular} 

& \begin{tabular}[c]{@{}c@{}}{}21.2528\\($\pm$0.3596)\end{tabular} 
& \begin{tabular}[c]{@{}c@{}}{}3003.4403\\($\pm$2.1556)\end{tabular} 
& \begin{tabular}[c]{@{}c@{}}{}0.1000\\($\pm$0.0000)\end{tabular} 
& \begin{tabular}[c]{@{}c@{}}{}0.0083\\($\pm$0.0000)\end{tabular} 
& \begin{tabular}[c]{@{}c@{}}{}0.1745\\($\pm$0.0002)\end{tabular} \\

\hline
& \multicolumn{5}{c|}{NLDAS-APCPSFC } 
& \multicolumn{5}{c}{ERA5-TISR}  \\
Method 
& in-AE $\downarrow$
& out-AE $\downarrow$ 
& Closeness $\downarrow$ 
& $R_S$ $\downarrow$ 
& $R_T$ $\downarrow$

& in-AE $\downarrow$
& out-AE $\downarrow$ 
& Closeness $\downarrow$ 
& $R_S$ $\downarrow$ 
& $R_T$ $\downarrow$\\
\hline
\textbf{FABLE} 
 
& \begin{tabular}[c]{@{}c@{}}{}69.3732\\($\pm$0.8647)\end{tabular} 
& \begin{tabular}[c]{@{}c@{}}{}\textcolor{red}{60.9266}\\($\pm$0.5791)\end{tabular} 
& \begin{tabular}[c]{@{}c@{}}{}\textcolor{red}{0.0454}\\($\pm$0.0001)\end{tabular} 
& \begin{tabular}[c]{@{}c@{}}{}\textcolor{blue}{0.0185}\\($\pm$0.0001)\end{tabular} 
& \begin{tabular}[c]{@{}c@{}}{}\textcolor{red}{0.0267}\\($\pm$0.0000)\end{tabular} 

& \begin{tabular}[c]{@{}c@{}}{}3.8683\\($\pm$0.0417)\end{tabular} 
& \begin{tabular}[c]{@{}c@{}}{}\textcolor{red}{34.1373}\\($\pm$0.2461)\end{tabular} 
& \begin{tabular}[c]{@{}c@{}}{}\textcolor{red}{0.0024}\\($\pm$0.0000)\end{tabular} 
& \begin{tabular}[c]{@{}c@{}}{}\textcolor{blue}{0.0003}\\($\pm$0.0000)\end{tabular} 
& \begin{tabular}[c]{@{}c@{}}{}0.0222\\($\pm$0.0000)\end{tabular} \\

ALA 
& \begin{tabular}[c]{@{}c@{}}{}\textcolor{blue}{63.0447}\\($\pm$0.8311)\end{tabular} 
& \begin{tabular}[c]{@{}c@{}}{}146.8263\\($\pm$1.5545)\end{tabular} 
& \begin{tabular}[c]{@{}c@{}}{}0.1258\\($\pm$0.0005)\end{tabular} 
& \begin{tabular}[c]{@{}c@{}}{}0.0347\\($\pm$0.0001)\end{tabular} 
& \begin{tabular}[c]{@{}c@{}}{}0.0521\\($\pm$0.0001)\end{tabular} 

& \begin{tabular}[c]{@{}c@{}}{}\textcolor{red}{1.5210}\\($\pm$0.0249)\end{tabular} 
& \begin{tabular}[c]{@{}c@{}}{}\textcolor{blue}{40.4142}\\($\pm$0.1138)\end{tabular} 
& \begin{tabular}[c]{@{}c@{}}{}\textcolor{blue}{0.0050}\\($\pm$0.0001)\end{tabular} 
& \begin{tabular}[c]{@{}c@{}}{}0.0004\\($\pm$0.0000)\end{tabular} 
& \begin{tabular}[c]{@{}c@{}}{}\textcolor{blue}{0.0214}\\($\pm$0.0000)\end{tabular} \\

TAAOWPF 
& \begin{tabular}[c]{@{}c@{}}{}\textcolor{red}{62.9704}\\($\pm$0.8302)\end{tabular} 
& \begin{tabular}[c]{@{}c@{}}{}143.9593\\($\pm$1.5400)\end{tabular} 
& \begin{tabular}[c]{@{}c@{}}{}0.2670\\($\pm$0.0013)\end{tabular} 
& \begin{tabular}[c]{@{}c@{}}{}0.0763\\($\pm$0.0003)\end{tabular} 
& \begin{tabular}[c]{@{}c@{}}{}0.0801\\($\pm$0.0002)\end{tabular} 

& \begin{tabular}[c]{@{}c@{}}{}\textcolor{blue}{1.6196}\\($\pm$0.0235)\end{tabular} 
& \begin{tabular}[c]{@{}c@{}}{}84.5149\\($\pm$0.1082)\end{tabular} 
& \begin{tabular}[c]{@{}c@{}}{}0.0128\\($\pm$0.0001)\end{tabular} 
& \begin{tabular}[c]{@{}c@{}}{}0.0004\\($\pm$0.0000)\end{tabular} 
& \begin{tabular}[c]{@{}c@{}}{}0.0245\\($\pm$0.0000)\end{tabular} \\

STPGD 

& \begin{tabular}[c]{@{}c@{}}{}64.2446\\($\pm$0.7891)\end{tabular} 
& \begin{tabular}[c]{@{}c@{}}{}\textcolor{blue}{128.0584}\\($\pm$5.6831)\end{tabular} 
& \begin{tabular}[c]{@{}c@{}}{}0.2264\\($\pm$0.0071)\end{tabular} 
& \begin{tabular}[c]{@{}c@{}}{}0.0763\\($\pm$0.0007)\end{tabular} 
& \begin{tabular}[c]{@{}c@{}}{}0.0712\\($\pm$0.0015)\end{tabular} 

& \begin{tabular}[c]{@{}c@{}}{}7.7767\\($\pm$0.3210)\end{tabular} 
& \begin{tabular}[c]{@{}c@{}}{}72.0045\\($\pm$0.3693)\end{tabular} 
& \begin{tabular}[c]{@{}c@{}}{}0.0065\\($\pm$0.0001)\end{tabular} 
& \begin{tabular}[c]{@{}c@{}}{}\textcolor{red}{0.0002}\\($\pm$0.0000)\end{tabular} 
& \begin{tabular}[c]{@{}c@{}}{}\textcolor{red}{0.0176}\\($\pm$0.0000)\end{tabular} \\

Noise Attack 
& \begin{tabular}[c]{@{}c@{}}{}69.0257\\($\pm$0.8436)\end{tabular} 
& \begin{tabular}[c]{@{}c@{}}{}5143.4710\\($\pm$4.1011)\end{tabular} 
& \begin{tabular}[c]{@{}c@{}}{}0.9273\\($\pm$0.0002)\end{tabular} 
& \begin{tabular}[c]{@{}c@{}}{}0.3478\\($\pm$0.0002)\end{tabular} 
& \begin{tabular}[c]{@{}c@{}}{}0.1643\\($\pm$0.0000)\end{tabular} 

& \begin{tabular}[c]{@{}c@{}}{}61.7027\\($\pm$0.1900)\end{tabular} 
& \begin{tabular}[c]{@{}c@{}}{}13700.4353\\($\pm$2.7649)\end{tabular} 
& \begin{tabular}[c]{@{}c@{}}{}0.1535\\($\pm$0.0000)\end{tabular} 
& \begin{tabular}[c]{@{}c@{}}{}0.0428\\($\pm$0.0000)\end{tabular} 
& \begin{tabular}[c]{@{}c@{}}{}0.0792\\($\pm$0.0000)\end{tabular} \\
FGSM 
& \begin{tabular}[c]{@{}c@{}}{}70.7596\\($\pm$0.8678)\end{tabular} 
& \begin{tabular}[c]{@{}c@{}}{}330.1197\\($\pm$0.9466)\end{tabular} 
& \begin{tabular}[c]{@{}c@{}}{}\textcolor{blue}{0.0753}\\($\pm$0.0000)\end{tabular} 
& \begin{tabular}[c]{@{}c@{}}{}\textcolor{red}{0.0114}\\($\pm$0.0000)\end{tabular} 
& \begin{tabular}[c]{@{}c@{}}{}\textcolor{blue}{0.0389}\\($\pm$0.0000)\end{tabular} 

& \begin{tabular}[c]{@{}c@{}}{}35.4651\\($\pm$0.8962)\end{tabular} 
& \begin{tabular}[c]{@{}c@{}}{}6334.0187\\($\pm$913.8126)\end{tabular} 
& \begin{tabular}[c]{@{}c@{}}{}0.0818\\($\pm$0.0001)\end{tabular} 
& \begin{tabular}[c]{@{}c@{}}{}0.0074\\($\pm$0.0002)\end{tabular} 
& \begin{tabular}[c]{@{}c@{}}{}0.0508\\($\pm$0.0002)\end{tabular} \\
\bottomrule
\end{tabular}
\end{table*}

\subsection{Performance Comparison} 


Table \ref{CLCRN-NLDAS} reports the performance comparison across four datasets, including NLDAS-TMP2M, NLDAS-APCPSFC, ERA5-T2M, and ERA5-TISR. Overall, FABLE demonstrates a favorable balance among attack faithfulness, closeness, and geospatio-temporal realisticness, as summarized in Table~\ref{tab:overall_ranking}.

\begin{table}[t]    
\centering
\caption{Overall ranking of attack methods across all datasets and  metrics. We count the number of first and second places for each method
across 4 datasets and 5 metrics.}
\label{tab:overall_ranking}
\setlength{\tabcolsep}{3pt}
\begin{tabular}{lcccccc}
\toprule
 & FABLE & ALA & STPGD & TAAOWPF & FGSM & Noise Attack \\
\midrule
\#First  & 11 & 5 & 2 & 1 & 1 & 0 \\
\#Second & 4  & 9 & 3 & 2 & 2 & 0 \\
Rank     & 1  & 2 & 3 & 4 & 4 & 6 \\
\bottomrule
\end{tabular}
\end{table}

In terms of {in-AE}, which measures the attack faithfulness within the targeted localized region, FABLE does not always achieve the lowest error compared with some gradient-based methods including ALA, TAAOWPF, and STPGD. However, these baselines typically incur larger {out-AE}, indicating that their perturbations unintentionally affect forecasts outside the targeted region. In contrast, FABLE consistently maintains smaller {out-AE}, suggesting that the perturbations remain localized and better preserve the original global forecast structure, thereby achieving improved overall faithfulness. On the other hand, Noise Attack and FGSM achieve the worst {in-AE} and {out-AE}. The former simply adds random noise without any optimization guidance, while the latter, being a one-step gradient-based attack, remains far from converging to the desired target.

For the {closeness} metric, FABLE achieves the best performance across all datasets, indicating that the adversarial inputs generated by FABLE are most similar to the original inputs and are therefore stealthier than those produced by other methods. Among the regular gradient-based attack methods, ALA—built on the Adam optimization framework, which adaptively adjusts perturbation magnitudes at each iteration—achieves better {closeness} than TAAOWPF and STPGD, both of which are based on the projected gradient descent (PGD) framework. Moreover, because STPGD constrains perturbations to a subset of victim locations, it generally attains better {closeness} than TAAOWPF across most datasets. 
It is not surprising that Noise Attack and FGSM achieve the relatively worst {closeness}, since the perturbations they introduce are either unguided or insufficiently refined. An exception occurs for the NLDAS-APCPSFC dataset, where the precipitation variable is highly skewed. In this case, regular gradient-based methods often require larger perturbations to shift forecasts between extreme and non-extreme values, and vice versa, compared to FGSM.

With respect to geospatio-temporal realisticness, FABLE achieves top-tier performance. In terms of spatial realisticness ($R_S$), FABLE attains the best overall performance, achieving one best and two second-best results, followed closely by ALA with one best and one second-best result. For temporal realisticness ($R_T$), FABLE also demonstrates superior performance, obtaining three best results, while ALA ranks next with two second-best results. The strong performance of FABLE can be attributed to its frequency-aware perturbation strategy. Unlike other baseline methods, which directly perturb the original forecasts, the selective perturbation of high-frequency coefficients that primarily influence localized variations enable FABLE to preserve the low-frequency structures that capture global spatio-temporal patterns.

\begin{table}[t!]
\footnotesize
\centering
\setlength{\tabcolsep}{1pt}
\caption{Runtime and memory usage comparison of different adversarial attack methods. 
}
\label{Runtime and Memory Usage}
\begin{tabular}{|l|c|c|c|c|}
\hline
Method & \multicolumn{2}{c|}{CLCRN (NLDAS-TMP2M)} & \multicolumn{2}{c|}{FourCastNet (ERA5-T2M)} \\
\cline{2-5}
& \begin{tabular}[c]{@{}c@{}}Total \\Runtime (s) \end{tabular}
& \begin{tabular}[c]{@{}c@{}}Peak \\Memory (MB)\end{tabular}
& \begin{tabular}[c]{@{}c@{}}Total \\Runtime (s) \end{tabular}
& \begin{tabular}[c]{@{}c@{}}Peak \\Memory (MB)\end{tabular} \\
\hline
\textbf{FABLE} & 345.87 & 77.81  & 112.53 & 477.57 \\
ALA            & 482.13 & 93.40  & 150.49 & 501.83 \\
TAAOWPF        & 498.58 & 76.14  & 149.74 & 483.83 \\
STPGD          & 490.24 & 96.30  & 158.25 & 502.07 \\
Noise Attack   & 143.71 & 70.67  & 40.43  & 244.09 \\
FGSM           & 2.11   & 76.12  & 1.07   & 493.01 \\
\hline
\end{tabular}
\end{table}

\subsection{Runtime and Memory Usage}

Since adversarial attacks may incur significant computational overhead, we evaluate the runtime efficiency and memory requirements of each method.
Table~\ref{Runtime and Memory Usage} reports the wall-clock runtime and peak GPU memory usage of {{FABLE}} and the baselines under two setups: CLCRN on NLDAS-TMP2M and FourCastNet on ERA5-T2M. All experiments were conducted for 500 attack iteration steps, except for FGSM, on a single NVIDIA L4 GPU (22.5 GB), with batch sizes of 48 (CLCRN) and 30 (FourCastNet). To avoid cross-batch memory interference, we run each batch independently. As presented in Table~\ref{Runtime and Memory Usage}, the non-learning baseline (Noise Attack) and one-step learning baseline (FGSM) have the shortest runtime, which is expected given their simplicity. However, this efficiency comes at the cost of weaker performance, as shown in Table~\ref{CLCRN-NLDAS}. Compared to other iterative learning methods (ALA, TAAOWPF, STPGD), {{FABLE}} demonstrates better runtime and lower peak memory usage. This efficiency benefits from excluding the perturbations on the \texttt{LLL} frequency sub-band, thereby reducing the number of parameters to be estimated.

\begin{table}[t!]
\footnotesize
\centering
\setlength{\tabcolsep}{4pt}
\caption{Detection performance of different anomaly detection techniques against adversarial inputs generated by various attack methods on the ERA5-T2M dataset using FourCastNet as the DLWF model.} 
\label{tab:adv_detection_results}
\begin{tabular}{|l|ccc|ccc|ccc|}
\hline
\multirow{2}{*}{\textbf{Method}} 
& \multicolumn{3}{c|}{\textbf{PCA-based}} 
& \multicolumn{3}{c|}{\textbf{IF}} 
& \multicolumn{3}{c|}{\textbf{LOF}} \\
\cline{2-10}
& Prec. & Rec. & F1 & Prec. & Rec. & F1 & Prec. & Rec. & F1 \\
\hline
FABLE        &0.66  &0.07  &0.12  &0.65      &0.95      &0.77  &0.64      &0.53      &0.58  \\
ALA          &0.73      &0.09      &0.16  &0.66      &0.97      &0.78  &0.66      &0.58      &0.62  \\
TAAOWPF      &0.78      &0.12      &0.21  &0.66      &0.98      &0.79  &0.67      &0.60      &0.63  \\
STPGD        &0.81      &0.14      &0.25  &0.65      &0.94      &0.77  &0.68      &0.62      &0.65  \\
Noise Attack &0.97      &1.00      &0.98  &0.67      &1.00      &0.80  &0.77      &1.00      &0.87  \\
FGSM         &0.91      &0.36      &0.51  &0.67      &1.00      &0.80  &0.76      &0.94      &0.84 \\
\hline
\end{tabular}
\end{table}

\subsection{Defense Evaluation}
In this experiment, our goal is to assess whether the adversarial inputs generated by different attack methods can be easily detected using three anomaly detection techniques: PCA-based~\cite{abdi2010principal}, Isolation Forest (IF)~\cite{liu2008isolation}, and Local Outlier Factor (LOF)~\cite{breunig2000lof}. 
Table~\ref{tab:adv_detection_results} presents the precision, recall, and F1-score of 
the anomaly detection methods when applied to the adversarial inputs generated by {{FABLE}} and other baseline attack methods. To enhance the detection capability, we follow prior work in~\cite{yao2023regularizing} and first decompose the  inputs into multi-scale wavelet components, from which statistical features are extracted for anomaly detection. Lower precision, recall, and F1-score indicate that adversarial samples are more “stealthy,” i.e., harder to distinguish from the clean samples. 

The results in Table~\ref{tab:adv_detection_results} show that {{FABLE}} consistently yields the lowest detection performance across all three anomaly detection techniques. This indicates that adversarial inputs generated by {{FABLE}} are more difficult to distinguish from non-adversarial inputs than those produced by the baseline methods. The results further suggest that detecting adversarial inputs requires more sophisticated techniques, such as Isolation Forest, rather than simpler methods such as PCA and LOF.


\subsection{Impact of Regularization Hyperparameter $\lambda$} 


Since the regularization hyperparameter $\lambda$ controls the trade-off among faithfulness, closeness, and realisticness of the attack (see Equation \eqref{FABLE_loss_function}), we analyze the impact of varying $\lambda$ on {{FABLE}}'s performance. As shown in Table~\ref{tab:lambda1_nldas}, increasing $\lambda$ consistently reduces $R_S$ and $R_T$, indicating improved geospatial–temporal realisticness. However, this improvement comes at the cost of degraded attack effectiveness, as reflected by the substantial increases in In-AE and Out-AE when $\lambda \ge 10^{-5}$. Meanwhile, closeness decreases monotonically with larger $\lambda$, suggesting that stronger regularization enforces smoother but less effective perturbations. These observations are consistent with our expectation since $\lambda$ constrains the magnitude of perturbations on wavelet coefficients. As $\lambda$ increases, the optimization increasingly suppresses perturbations on high-frequency wavelet coefficients, which in turn, 
produces smoother patterns that better preserve the intrinsic spatial and temporal correlations, leading to improved $R_S$ and $R_T$.




\begin{table}[t!]
\centering
\setlength{\tabcolsep}{1.5pt}
\caption{Impact of regularization strength $\lambda$ on FABLE's performance on CLCRN model for TMP2M dataset.} 
\label{tab:lambda1_nldas}
\begin{tabular}{|c|ccccc|}
\hline
\textbf{$\lambda$} & In-AE $\downarrow$ & Out-AE $\downarrow$ & Closeness $\downarrow$ & \textbf{$R_S$} $\downarrow$ & \textbf{$R_T$} $\downarrow$ \\
\hline
0       & 0.3786 & 1.3597 & 0.0301 & 0.1457 & 0.0495 \\
1e-6    & 0.3934 & 1.3561 & 0.0135 & 0.0639 & 0.0260 \\
1e-5    & 0.4591 & 1.1682 & 0.0025 & 0.0075 & 0.0064 \\
1e-4    & 0.5749 & 1.4388 & 0.0006 & 0.0013 & 0.0013 \\
1e-3    & 2.0501 & 4.2867 & 0.0004 & 0.0001 & 0.0011 \\
\hline
\end{tabular}
\end{table}

\subsection{Impact of Wavelet Bases} 

This section investigates how the choice of wavelet bases affects the attack behavior of {{FABLE}}. For this experiment, we use the Daubechies (db) wavelet, a family of orthogonal wavelets that includes Haar (db1) and other wavelet bases.  Table~\ref{tab:Wavelet Family Sensitivity Analysis} presents the results as the order of the \textit{Daubechies} wavelet increases from db1 to db6. While there is a slight decreasing trend in {In-AE} and a more noticeable increasing trend in {Out-AE}, varying the order of the Daubechies wavelet does not significantly affect closeness or geospatio-temporal realisticness of the adversarial inputs.

\begin{table}[t!]
\centering
\caption{Impact of varying order of Daubechies wavelet on FABLE's  performance on CLCRN model for TMP2M dataset.}
\label{tab:Wavelet Family Sensitivity Analysis}
\begin{tabular}{|l|ccccc|}
\hline
\textbf{} &  In-AE $\downarrow$ & Out-AE $\downarrow$ & Closeness $\downarrow$ & \textbf{$R_S$} $\downarrow$ & \textbf{$R_T$} $\downarrow$ \\
\hline
Haar (db1) & 0.3934 & 1.3561 & 0.0135 & 0.0639 & 0.0260 \\
db2  & 0.3902  & 1.6689  & 0.0129  & 0.0592  & 0.0247  \\
db4  & 0.3919  & 1.8371  & 0.0143  & 0.0659  & 0.0269  \\
db6  & 0.3910  & 1.8355  & 0.0135  & 0.0615  & 0.0266  \\
\hline
\end{tabular}
\end{table}




\section{Discussion and Implications}
\label{sec:Discussions and Ethics}

FABLE can potentially be misused to manipulate forecasts in ways that may deceive downstream decision-making systems. One example, illustrated in Figure~\ref{fig:FABLE_case_study}, involves perturbing precipitation forecasts in agriculturally important regions, which could influence financial decisions such as crop insurance or futures trading. 
That said, FABLE can also be used defensively for robustness testing. For example, by incorporating FABLE-generated samples with the true model forecasts into training, this can help improve model robustness. 

\section{Conclusion and future works}
\label{sec:Conclusion}

In this work, we explore adversarial attacks on DLWF models and highlight the key challenges of generating stealthy adversarial samples. We introduce metrics to quantify the spatiotemporal realisticness of adversarial samples and propose {{FABLE}}, a novel framework that perturbs wavelet-decomposed coefficients instead of the raw inputs. By focusing on higher-frequency components, this facilitates its generation of adversarial samples that satisfy the {realisticness}, {closeness}, and {faithfulness} criteria. 
Our work does not target a specific operational forecasting system. Instead, it aims to highlight potential vulnerabilities and motivate the development of more robust DLWF models, as well as security-aware practices across the forecasting pipeline. 

Future directions of this work include exploring higher-level decompositions and other families of wavelet bases for generating adversarial inputs, as well as extending {{FABLE}} to multivariate geospatio-temporal settings, where ensuring physical consistency across correlated variables remains a key challenge. Another direction is to investigate the transferability of adversarial inputs across different DLWF models and their impact on downstream applications.




\bibliographystyle{IEEEtran}
\bibliography{yue}

\clearpage

\appendices
\section{Haar Wavelet Transform}
\label{appendix:wavelet_transform}

Section \ref{sec:wavelet} presents the decomposition of a one-dimensional signal using the discrete wavelet transform with the Haar wavelet basis. In this section, we extend the discussion to the two-dimensional and three-dimensional cases.

\subsection{2-dimensional Haar Wavelet Decomposition}

For a two-dimensional data, $\mathbf{X} \in \mathbb{R}^{r \times c}$, its decomposition is performed in two steps. For brevity, we denote the approximation component of the decomposition as $L$ and its detail component as $H$. 

First, a one-dimensional decomposition is applied to each \emph{column} $k_3^*$ of $\mathbf{X}_{i,j}$ along the rows, producing intermediate low-frequency ($C^L_{k_2, k_3^*}$) and high-frequency ($C^H_{k_2, k_3^*}$) coefficients:
\begingroup
\begin{align}
C^L_{k_2, k_3^*} &= \frac{X_{2k_2-1, k_3^*} + X_{2k_2, k_3^*}}{\sqrt{2}}, 
\quad k_2 \in \{1, 2, \dots, r/2 \},\nonumber \\
C^H_{k_2, k_3^*} &= \frac{X_{2k_2-1, k_3^*} - X_{2k_2, k_3^*}}{\sqrt{2}},
\quad k_2 \in \{1, 2, \dots,  r/2 \}. \nonumber
\end{align}
\endgroup

Next, a secondary one-dimensional decomposition is applied to each \emph{row} of $C^L_{k_2, k_3^*}$ and $C^H_{k_2, k_3^*}$ along the columns, producing the final approximation ($C^{LL}_{k_2,k_3}$) and detail coefficients ($C^{LH}_{k_2,k_3}$, $C^{HL}_{k_2,k_3}$, and $C^{HH}_{k_2,k_3}$), where:
\begingroup\small
\begin{align}
C^{LL}_{k_2,k_3} &= \frac{C^L_{k_2, 2k_3 - 1} + C^L_{k_2, 2k_3}}{\sqrt{2}}, \quad
C^{LH}_{k_2,k_3} = \frac{C^L_{k_2, 2k_3 - 1} - C^L_{k_2, 2k_3}}{\sqrt{2}}, \nonumber\\
C^{HL}_{k_2,k_3} &= \frac{C^H_{k_2, 2k_3 - 1} + C^H_{k_2, 2k_3}}{\sqrt{2}}, \quad 
C^{HH}_{k_2,k_3} = \frac{C^H_{k_2, 2k_3 - 1} - C^H_{k_2, 2k_3}}{\sqrt{2}},\nonumber
\end{align}
\endgroup
where $k_2 \in \{1, 2, \dots, \lfloor r/2 \rfloor\}$ and $k_3 \in \{1, 2, \dots, \lfloor c/2 \rfloor\}$. 



\subsection{3-dimensional Haar Wavelet Decomposition}

The three-dimensional Haar wavelet decomposition extends the
two-dimensional case by sequentially applying the one-dimensional
Haar decomposition along three dimensions: the temporal, row (longitude), and column (latitude) axes of the data tensor
$\mathbf{X} \in \mathbb{R}^{(\alpha+1)\times r \times c}$.

Following the same notation as the 2-dimensional case,
we denote the low-frequency (approximation) component by $L$
and the high-frequency (detail) component by $H$.
Applying the transform along three dimensions produces
eight sub-band coefficients corresponding to different
combinations of low-pass and high-pass filtering: 
\texttt{LLL}, \texttt{LLH}, \texttt{LHL}, \texttt{LHH}, \texttt{HLL}, \texttt{HLH}, \texttt{HHL}, \texttt{HHH}, where, for example, the coefficient at \texttt{LLL} sub-band can be expressed as
\[
C^{LLL}_{k_1,k_2,k_3}
=
\frac{1}{\sqrt{8}}
\sum_{n_1,n_2,n_3=0}^{1}
X_{2k_1-n_1,\;2k_2-n_2,\;2k_3-n_3}.
\]
Similarly, the remaining sub-band coefficients
(\texttt{LLH}, \texttt{LHL}, \texttt{LHH}, \texttt{HLL}, \texttt{HLH}, \texttt{HHL}, \texttt{HHH})
can be expressed.

In all, all sub-band coefficients can be written
in a unified form. Let $f=f_1f_2f_3$ denote the combination
of filters applied along the temporal, row, and column
dimensions, where $f_d \in \{L,H\}$.
Then the wavelet coefficient associated with sub-band $f$
at location $(k_1,k_2,k_3)$ can be expressed as
\begin{align}
    C^{\mathbf{f}}_{k_1, k_2, k_3} 
    &= \sum_{n_1=0}^{1} \sum_{n_2=0}^{1} \sum_{n_3=0}^{1} 
    \left[\frac{1}{\sqrt{2}} (-1)^{\xi_1 n_1}\right] \left[\frac{1}{\sqrt{2}} (-1)^{\xi_2 n_2}\right] \nonumber \\
    & \hspace{0.6cm}
    \times \left[\frac{1}{\sqrt{2}} (-1)^{\xi_3 n_3}\right] 
    \mathbf{X}_{2k_1 - n_1, 2k_2 - n_2, 2k_3 - n_3} \nonumber \\
    &= \frac{1}{\sqrt{8}} \sum_{n_1, n_2, n_3=0}^{1}  
    (-1)^{\boldsymbol{\xi}_{\mathbf{f}}^T \mathbf{n}} 
    \mathbf{X}_{2k_1 - n_1, 2k_2 - n_2, 2k_3 - n_3}, \nonumber
\end{align}

\section{Proof of Theorem~\ref{upper_bound_comparison}}
\label{proof_of_theorem_upper_bound_comparison}

Consider a 1-D signal $X_t$ that has been standardized to zero mean and unit standard deviation. The lag-$l$ autocorrelation of the standardized signal, $\rho(l)$, is defined as
\begin{align}
    \rho(l)&= \sum_{t=1}^{T-l} X_{l+t}X_t  \nonumber\\
    &= X_{l+1}X_{1}+X_{l+2}X_{2}+...+X_{T-1}X_{T-1-l}+X_{T}X_{T-l}\nonumber \\
    &= \begin{cases}
        \sum_{n=0}^1\sum_{k=\frac{l+2}{2}}^{\frac{T}{2}}X_{2k-n}X_{2k-n-l} & l \textrm{ is even}\\
        \sum_{n=0}^1\sum_{k=\frac{l+3}{2}}^{\frac{T}{2}}X_{2k-n}X_{2k - (1-n)- l+(-1)^{n}} & l \textrm{ is odd}\\
        \hspace{2cm} +X_{l+1}X_{1} & 
    \end{cases}   
    \label{eqn:odd_even}
\end{align}
%
%
where the signal length $T$ is assumed to be an even number. Note that for the odd lag $l$, the additional term $X_{l+1}X_{1}$ arises when $k=\frac{l+1}{2}$ and $n=0$, since in this case, $$X_{2k-n}X_{2k-(1-n)-l+(-1)^{n}} = X_{l+1}X_1$$ There is no corresponding term for $k=\frac{l+1}{2}$ and $n=1$ as this would involve $X_lX_0$, which lies outside the index set for $t$.

The preceding equation allows us to express the lag-$l$ autocorrelation in terms of the Haar approximation and detail components by mapping the time index $t$ to the wavelet index pair $(k,n)$, where $t = 2k-n$, $k \in \{1,2,\cdots,\frac{T}{2}\}$, and $n \in \{0, 1\}$. To avoid confusion with subscripts, we denote \( X_t \) as \( X(t) \) for the remainder of this section. The level-one Haar decomposition of the signal $X(t)$ 
can be written as follows:
\begingroup
\begin{align}
X(k,n) \equiv X(2k - n) &= \frac{C_k^L}{\sqrt{2}} + \frac{(-1)^{1-n} C_k^H}{\sqrt{2}} \nonumber \\
&= X_A(k,n) + X_D(k,n)
\label{proof_decomposition}
\end{align}
\endgroup
where $X_{A}(k,n) = \frac{C_k^L}{\sqrt{2}}$ and  $X_D(k,n) = \frac{(-1)^{1-n}\,C_k^H}{\sqrt{2}}$ are the respective approximation and detail components. 

The autocorrelation at lag-$l$ in Equation~\eqref{eqn:odd_even} can be expressed in terms of the wavelet index pair $(k,n)$ as follows:
\begin{align}
\small
& \ \rho(l) \nonumber
\\
&= \begin{cases}
    \sum_{n=0}^1 \sum_{k=\frac{l+2}{2}} X(k,n) X(k-\frac{l}{2},n) & l \text{ is even}\\
    \sum_{n=0}^1 \sum_{k=\frac{l+3}{2}} X(k,n) X(k-\frac{l-(-1)^n}{2},1-n) & l \text{ is odd}  \\
    \hspace{2.1cm} + X(\frac{l+1}{2},0) X(1,1) & 
\end{cases} 
\normalsize
\label{eqn:autocorr}
\end{align}

Let $\rho_A(l)$ and $\rho_D(l)$ be the autocorrelations at lag $l$ for the approximation-perturbed ($X_A'$) and detail-perturbed ($X_D'$) signals, respectively, where
\begin{align}
    X_A'(k,n) &= \frac{C_k^L+ \delta_A(k)}{\sqrt{2}} + \frac{(-1)^{1-n} C_k^H}{\sqrt{2}} \nonumber\\
    &= X(k,n) + \Delta X_A(k,n) \label{eqn:perturbed_A}\\
    X_D'(k,n) &= \frac{C_k^L}{\sqrt{2}} + \frac{(-1)^{1-n} (C_k^H+ \delta_D(k))}{\sqrt{2}} \nonumber\\
    &= X(k,n) + \Delta X_D(k,n),
    \label{eqn:perturbed_D}
\end{align}
where $\Delta X_A(k,n) = \frac{\delta_A(k)}{\sqrt{2}}$ and $\Delta X_D(k,n) = \frac{(-1)^{1-n} \delta_D(k)}{\sqrt{2}}$ are the perturbations added to $C_k^L$ and $C_k^H$, respectively. 

\setcounter{theorem}{0}
\begin{theorem}
Assuming
$
\sum_{k=1}^{T/2} |C^{\texttt{H}}_k| \leq \sum_{k=1}^{T/2} |C^{\texttt{L}}_k|$ and the perturbation magnitudes $\{\delta_A(k)\}$ and $\{\delta_D(k)\}$ are both upper bounded by some constant $\epsilon^*$, i.e., $\forall k: |\delta_A(k)|\leq \epsilon^*$ and $|\delta_D(k)|\leq \epsilon^*$,  it follows that
$$
\sup_{\delta_D} \sum_{l=0}^{T-1} |\rho_D(l) - \rho(l)|
\leq
\sup_{\delta_A} \sum_{l=0}^{T-1} |\rho_A(l) - \rho(l)| 
.
$$
\end{theorem}


\begin{proof}
Let \( \{\delta_A(k)\} \) and \( \{\delta_D(k)\} \) be the perturbations added to $\{C_k^L\}$ and $\{C_k^H\}$,
respectively, as shown in Equations~\eqref{eqn:perturbed_A} and \eqref{eqn:perturbed_D}. 
We need to compute: (1) $\sum_{l=0}^{T-1} |\rho_A(l) - \rho(l)|$ and (2) $\sum_{l=0}^{T-1} |\rho_D(l) - \rho(l)|$ and compare their upper bounds.

\vspace{0.3cm}
\noindent \textbf{Step 1: Computing Upper Bound of $\sum_{l=0}^{T-1} |\rho_A(l) - \rho(l)|$} \vspace{0.2cm}
 

Replacing Equation \eqref{eqn:perturbed_A} into the definition of autocorrelation function in Equation~\eqref{eqn:autocorr}, we have 
\begin{eqnarray}
&& \ \rho_{A}(l)-\rho(l)  \nonumber \\
&=&\sum_n\sum_k \bigg\{\left[X(k,n) + \Delta X_A(k,n)\right] \nonumber\\
&& \hspace{1.5cm} \times \left[X\left(k-\frac{l}{2},n\right) + \Delta X_A\left(k-\frac{l}{2},n\right)\right]\nonumber \\
&& \hspace{1cm} - \ X(k,n)X\left(k-\frac{l}{2},n\right)\bigg\}
\nonumber \\
&=& 
\underbrace{\sum_n\sum_k X(k,n)\,\Delta X_A\left(k - \frac{l}{2}, n\right)}_{\text{(I)}} \nonumber \\
&&  + \underbrace{\sum_n\sum_k \Delta X_A(k,n)\,X\left(k-\frac{l}{2}, n\right)}_{\text{(II)}} \nonumber\\
&&  + \underbrace{\sum_n\sum_k \Delta X_A(k,n)\,\Delta X_A\left(k -\frac{l}{2}, n\right)}_{\text{(III)}}.
\label{eqn:terms_case1}
\end{eqnarray}

\noindent \textbf{(1) For term (I).} 
Using Equation~\eqref{proof_decomposition}, we split $X(k,n)$ into 
$X_A(k,n) + X_D(k,n)$, yielding
\begin{align}
\text{(I)} 
&= \sum_{n}\sum_k X_A(k,n)\,\Delta X_A\left(k - \frac{l}{2}, n\right) \nonumber\\ 
& + \sum_n \sum_k X_D(k,n)\,\Delta X_A\left(k - \frac{l}{2},n \right). \nonumber
\end{align}
When $l$ is even, the first term on the right-hand side of (I) can be simplified as follows after expanding the sum over $n$:
\begin{align}
    &\sum_{n=0}^1 \sum_{k=\frac{l+2}{2}}^{\frac{T}{2}} X_A(k,n) \Delta X_A\left(k-\frac{l}{2}, n\right) \nonumber \\
    =&\sum_{k=\frac{l+2}{2}}^{\frac{T}{2}} \bigg\{X_A(k,0)\Delta X_A\left(k-\frac{l}{2},0\right) \nonumber\\
    & \hspace{0.7cm} + \  X_A(k,1)\Delta X_A\left(k-\frac{l}{2},1\right) \bigg\}\nonumber \\
    =&  \sum_{k=\frac{l+2}{2}}^{\frac{T}{2}} \left\{\frac{C_k^L}{\sqrt{2}}\frac{\delta_A(k-\frac{l}{2})}{\sqrt{2}}+ \frac{C_k^L}{\sqrt{2}}\frac{\delta_A(k-\frac{l}{2})}{\sqrt{2}} \right\} \nonumber\\\ 
    =& \sum_{k=\frac{l+2}{2}}^{\frac{T}{2}} C_k^L \delta_A\left(k-\frac{l}{2}\right) \nonumber
`\end{align}
while the second term on the right-hand side of (I) vanishes upon summing over $n$:
\begin{align}
    & \sum_{n=0}^1\sum_{k=\frac{l+2}{2}}^{\frac{T}{2}} X_D(k,n)\,\Delta X_A\left(k - \frac{l}{2}, n\right) \nonumber \\ 
    =& \sum_{k=\frac{l+2}{2}}^{\frac{T}{2}} \bigg\{ X_D(k,0)\Delta X_A\left(k - \frac{l}{2}, 0\right) \nonumber \\
    & \hspace{0.7cm} + \ X_D(k,1)\,\Delta X_A\left(k - \frac{l}{2}, 1\right)\bigg\}\nonumber \\
    =& \sum_{k=\frac{l+2}{2}}^{\frac{T}{2}} \left\{\frac{(-1)^1 C_k^H}{\sqrt{2}}\frac{\delta_A(k-\frac{l}{2})}{\sqrt{2}} + 
    \frac{(-1)^0 C_k^H}{\sqrt{2}} \frac{\delta_A(k-\frac{l}{2})}{\sqrt{2}}\right\}\nonumber \\
    =& \ 0 \nonumber
\end{align}

In contrast, when $l$ is odd, the first-term on the right-hand side of (I) reduces to the following using Equation~\eqref{eqn:autocorr}:
\begin{align}
    & \sum_{n=0}^1 \sum_{k=\frac{l+3}{2}}^{\frac{T}{2}} X_A(k,n) \Delta X_A\left(k-\frac{l-(-1)^n}{2},1-n\right)  \nonumber\\
    & \hspace{1.5cm} 
    + X_A\left(\frac{l+1}{2},0\right)\Delta X_A\left(1,1\right) \nonumber\\
    &= \sum_{k=\frac{l+3}{2}}^{\frac{T}{2}} \bigg\{ X_A(k,0) \Delta X_A\left(k-\frac{l-1}{2},1\right)  \nonumber\\
    & \hspace{0.7cm} 
    + \sum_{k=\frac{l+3}{2}}^{\frac{T}{2}} X_A(k,1) \Delta X_A\left(k-\frac{l+1}{2},0\right)\bigg\}  
    + \frac{C_{\frac{l+1}{2}}^L \delta_A(1)}{2} \nonumber\\
    =& \sum_{k=\frac{l+3}{2}}^{\frac{T}{2}} \left\{\frac{C_k^L}{\sqrt{2}}\frac{\delta_A(k-\frac{l-1}{2})}{\sqrt{2}}+
    \frac{C_k^L}{\sqrt{2}}\frac{\delta_A(k-\frac{l+1}{2})}{\sqrt{2}} \right\} 
    + \frac{C_{\frac{l+1}{2}}^L \delta_A(1)}{2} \nonumber\\
    =& \sum_{k=\frac{l+3}{2}}^{\frac{T}{2}} \frac{C_k^L}{2}\left[\delta_A\left(k-\frac{l-1}{2}\right)+\delta_A\left(k-\frac{l+1}{2}\right)\right] + \frac{C_{\frac{l+1}{2}}^L \delta_A(1)}{2} \nonumber
\end{align}

\noindent while the second term on the right-hand side of (I) when $l$ is odd can be written as follows:
\begin{align}
& \sum_{n=0}^1 \sum_{k=\frac{l+3}{2}}^{\frac{T}{2}}X_D(k,n)\Delta X_A\left(k-\frac{l-(-1)^n}{2},1-n\right) \nonumber \\
& \hspace{1.5cm} + X_D\left(\frac{l+1}{2},0 \right)\Delta X_A(1,1)\nonumber \\
=& \sum_{k=\frac{l+3}{2}}^{\frac{T}{2}} \bigg\{
X_D(k,0)\Delta X_A\left(k-\frac{l-1}{2},1\right)
\nonumber \\
&\hspace{0.5cm} + X_D(k,1)\Delta X_A\left(k-\frac{l+1}{2},0\right) \bigg\} 
+\frac{(-1)^1 C_{\frac{l+1}{2}}^H\delta_A(1)}{2} \nonumber\\
=& \sum_{k=\frac{l+3}{2}}^{\frac{T}{2}} \left\{\frac{(-1)^{1} C_k^H}{\sqrt{2}}\frac{\delta_A(k-\frac{l-1}{2})}{\sqrt{2}} + 
\frac{(-1)^{0}C_k^H}{\sqrt{2}}\frac{\delta_A(k-\frac{l+1}{2})}{\sqrt{2}} \right\}\nonumber\\
& \hspace{0.8cm} -\frac{C_{\frac{l+1}{2}}^H\delta_A(1)}{2} \nonumber\\
=& \sum_{k=\frac{l+3}{2}}^{\frac{T}{2}} \frac{C_k^H}{2}\left[\delta_A\left(k-\frac{l+1}{2}\right)-\delta_A\left(k-\frac{l-1}{2}\right)\right] \nonumber\\
&\hspace{0.8cm} -\frac{C_{\frac{l+1}{2}}^H\delta_A(1)}{2} \nonumber.
\end{align}


Note the difference in the first term of (I) depending on whether $l$ is even or odd. When $l$ is even, decomposing the sum into its $n=0$ and $n=1$ contributions reveals that, for each $k$, both contributions involve the same perturbation, $\delta_A\left(k-\frac{l}{2}\right)$. In contrast, when $l$ is odd, the two contributions involve different shifts of $k$ in the perturbed approximation coefficients, namely, $\delta_A\left(k-\frac{l-1}{2}\right)$ and $\delta_A\left(k-\frac{l+1}{2}\right)$. In addition, there is an extra term, $\delta_A(1)$, arising from the initial condition $k = \frac{l+1}{2}$ and $n=0$. 

Combining these expressions together, term (I) appearing in Equation~\eqref{eqn:terms_case1} can thus be summarized as follows: 
\begin{align}
&\text{(I)}= 
\begin{cases} 
 \sum_{k=\frac{l+2}{2}}^{\frac{T}{2}} C_k^L \delta_A\left(k-\frac{l}{2}\right) & \text{if l is even}  \\
  \sum_{k=\frac{l+3}{2}}^{\frac{T}{2}}\frac{C_k^L+C_k^H}{2} \delta_A\left(k-\frac{l+1}{2}\right) &\text{if l is odd}\\
  \hspace{0.5cm} +\sum_{k=\frac{l+3}{2}}^{\frac{T}{2}}\frac{C_k^L-C_k^H}{2} \delta_A\left(k-\frac{l-1}{2}\right)\\
    \hspace{0.5cm}+ \ \frac{\delta_A(1)}{2}\left(C^L_{\frac{l+1}{2}} -C^H_{\frac{l+1}{2}}\right) &
\end{cases}
\end{align}

\vspace{0.2cm}
\textbf{(2) For term (II),} we again split $X\left(k-\frac{l}{2},n\right)$ into its approximation and detail components to obtain the following:
\begin{align}
    \text{(II)}&=\sum_{n} \sum_{k} \Delta X_A(k,n)X_A\left(k-\frac{l}{2},n\right)  \nonumber \\
    &+  \sum_{n}\sum_{k} \Delta X_A(k,n) X_D\left(k-\frac{l}{2},n\right). \nonumber
\end{align}

When $l$ is even, the first term on the right-hand side of (II) can be simplified by expanding the sum over $n$ as follows:
\begin{align}
    &\sum_{n=0}^1 \sum_{k=\frac{l+2}{2}}^{\frac{T}{2}} \Delta X_A(k,n) X_A\left(k-\frac{l}{2},n\right) \nonumber \\
    =& \sum_{k=\frac{l+2}{2}}^{\frac{T}{2}} \bigg\{\Delta X_A(k,0) X_A\left(k-\frac{l}{2},0\right) \nonumber \\
    &\hspace{0.8cm} + \Delta X_A(k,1) X_A\left(k-\frac{l}{2},1\right)\bigg\}\nonumber \\
    =& \sum_{k=\frac{l+2}{2}}^{\frac{T}{2}} \bigg\{\frac{\delta_A(k)}{\sqrt{2}} \frac{C_{k-\frac{l}{2}}^L}{\sqrt{2}}+ \frac{\delta_A(k)}{\sqrt{2}}\frac{C_{k-\frac{l}{2}}^L}{\sqrt{2}} \bigg\}\nonumber \\
    =& \sum_{k=\frac{l+2}{2}}^{\frac{T}{2}}\delta_A(k)C^L_{k-\frac{l}{2}}\nonumber 
\end{align}
while the second term of (II) vanishes after expanding over $n$: 
\begin{align}
    &\sum_{n=0}^1 \sum_{k=\frac{l+2}{2}}^{\frac{T}{2}}\Delta X_A(k,n) X_D\left(k-\frac{l}{2},n\right) \nonumber \\
    =& \sum_{k=\frac{l+2}{2}}^{\frac{T}{2}}\Delta X_A(k,0) X_D\left(k-\frac{l}{2},0\right) \nonumber \\
    &\hspace{0.8cm} + \Delta X_A(k,1) X_D\left(k-\frac{l}{2},1\right) \nonumber \\
    =&\sum_{k=\frac{l+2}{2}}^{\frac{T}{2}}\frac{\delta_A(k)}{\sqrt{2}} \frac{(-1)^1 C^H_{k-\frac{l}{2}}}{\sqrt{2}} + \sum_{k=\frac{l+2}{2}}^{\frac{T}{2}}\frac{\delta_A(k)}{\sqrt{2}} \frac{(-1)^0C^H_{k-\frac{l}{2}}}{\sqrt{2}}\nonumber\\
    =& \ 0 \nonumber 
\end{align}

When $l$ is odd, the first term on the right-hand side of (II) can be simplified as follows:
\begin{align}
    & \sum_{n=0}^1 \sum_{k=\frac{l+3}{2}}^{\frac{T}{2}} \Delta X_A(k,n) X_A\left(k-\frac{l-(-1)^n}{2},1-n\right) \nonumber\\ 
    & \hspace{1cm} + \Delta X_A\left(\frac{l+1}{2},0\right) X_A(1,1)\nonumber \\
    =& \sum_{k=\frac{l+3}{2}}^{\frac{T}{2}}\bigg\{\Delta X_A(k,0) X_A\left(k-\frac{l-1}{2},1\right) \nonumber \\
    & \hspace{0.5cm} +\Delta X_A(k,1)X_A\left(k-\frac{l+1}{2},0\right)\bigg\} +\frac{\delta_A(\frac{l+1}{2})}{\sqrt{2}}\frac{ C^L_{1}}{\sqrt{2}}\nonumber \\
    =& \sum_{k=\frac{l+3}{2}}^{\frac{T}{2}} \left\{\frac{\delta_A(k)}{\sqrt{2}}\frac{C_{k-\frac{l-1}{2}}^L}{\sqrt{2}}+
    \frac{\delta_A(k)}{\sqrt{2}}\frac{C_{k-\frac{l+1}{2}}^L}{\sqrt{2}} \right\}
    +\frac{\delta_A(\frac{l+1}{2})}{\sqrt{2}}\frac{ C^L_{1}}{\sqrt{2}}\nonumber\\
    =& \sum_{k=\frac{l+3}{2}}^{\frac{T}{2}}\frac{\delta_A(k)}{2}\left[C_{k-\frac{l-1}{2}}^L+C_{k-\frac{l+1}{2}}^L\right]+\frac{ C^L_{1}\delta_A(\frac{l+1}{2})}{2} \nonumber
\end{align}

while the second term in (II) can be simplified to
\begin{align}
    &\sum_{n=0}^1 \sum_{k=\frac{l+3}{2}}^{\frac{T}{2}} \Delta X_A(k,n) X_D\left(k-\frac{l-(-1)^n}{2},1-n\right) \nonumber \\
    &  \hspace{1cm} + \Delta X_A\left(\frac{l+1}{2},0\right) X_D(1,1) \nonumber \\
    =& \sum_{k=\frac{l+3}{2}}^{\frac{T}{2}} \bigg\{\Delta X_A(k,0) X_D\left(k-\frac{l-1}{2},1\right) \nonumber \\
    & \hspace{0.5cm} +\Delta X_A(k,1)X_D\left(k-\frac{l+1}{2},0\right) \bigg\} + \frac{\delta_A(\frac{l+1}{2})}{\sqrt{2}} \frac{(-1)^{0}C_1^H}{\sqrt{2}} \nonumber\\
    =& \sum_{k=\frac{l+3}{2}}^{\frac{T}{2}} \left\{\frac{\delta_A(k)}{\sqrt{2}} \frac{(-1)^{0}C_{k-\frac{l-1}{2}}^H}{\sqrt{2}} +
    \frac{\delta_A(k)}{\sqrt{2}} \frac{(-1)^{1}C_{k-\frac{l+1}{2}}^H}{\sqrt{2}} \right\} \nonumber \\
    & \hspace{0.5cm} + \frac{\delta_A(\frac{l+1}{2})C_1^H}{2} \nonumber\\
    = & \sum_{k=\frac{l+3}{2}}^{\frac{T}{2}}\frac{\delta_A(k)}{2}\left[ C_{k-\frac{l-1}{2}}^H-C_{k-\frac{l+1}{2}}^H \right]+ \frac{C_1^H \delta_A(\frac{l+1}{2})}{2}\nonumber
\end{align}

Putting them together, the second term in Equation~\eqref{eqn:terms_case1} can be summarized as follows:
\begin{equation*}
\text{(II)} =
\begin{cases}
\sum_{k=\frac{l+2}{2}}^{\frac{T}{2}}  C^L_{k-\frac{l}{2}} \delta_A(k) & \text{if $l$ is even}, \\
\sum_{k=\frac{l+3}{2}}^{\frac{T}{2}} \frac{C_{k-\frac{l-1}{2}}^L+C_{k-\frac{l-1}{2}}^H}{2}\delta_A(k)  & \text{if $l$ is odd}. \\
\hspace{0.5cm}+ \sum_{k=\frac{l+3}{2}}^{\frac{T}{2}} \frac{C_{k-\frac{l+1}{2}}^L-C_{k-\frac{l+1}{2}}^H}{2}\delta_A(k) \\
\hspace{0.5cm} + \ \frac{\delta_A(\frac{l+1}{2})}{2} \left(C^L_{1} + C^H_{1}\right) &  
\end{cases}
\end{equation*}

\vspace{0.3cm}
\textbf{(3) For term (III)}, the expression can be simplified as follows when $l$ is even after expanding the sum over $n$:
\begin{align}
    &\sum_{n=0}^1 \sum_{k=\frac{l+2}{2}}^{\frac{T}{2}} \Delta X_A(k,n)\,\Delta X_A\left(k - \frac{l}{2},n\right) \nonumber \\
    =& \sum_{k=\frac{l+2}{2}}^{\frac{T}{2}} \bigg\{\Delta X_A(k,0) \Delta X_A\left(k - \frac{l}{2},0\right)\nonumber \\
    & \hspace{0.8cm} + \Delta X_A(k,1)\,\Delta X_A\left(k - \frac{l}{2},1\right)\bigg\}\nonumber\\
    =& \sum_{k=\frac{l+2}{2}}^{\frac{T}{2}} \left\{\frac{\delta_A(k)}{\sqrt{2}}\frac{\delta_A(k-\frac{l}{2})}{\sqrt{2}}+
    \frac{\delta_A(k)}{\sqrt{2}}\frac{\delta_A(k-\frac{l}{2})}{\sqrt{2}}\right\} \nonumber \\
    =& \sum_{k=\frac{l+2}{2}}^{\frac{T}{2}}\delta_A(k) \delta_A(k-\frac{l}{2}) \nonumber 
\end{align}

However, when $l$ is odd, expanding the sum over $n$ leads to the following expression:
\begin{align}
    & \sum_{n=0}^1\sum_{k=\frac{l+3}{2}}^{\frac{T}{2}} \Delta X_A(k,n)\,\Delta X_A\left(k-\frac{l-(-1)^n}{2},1-n\right)  \nonumber \\
    & \hspace{1cm} + \Delta X_A\left(\frac{l+1}{2},0\right) \Delta X_A(1,1)\nonumber \\
    =& \sum_{k=\frac{l+3}{2}}^{\frac{T}{2}} \bigg\{\Delta X_A(k,0)\Delta X_A\left(k-\frac{l-1}{2},1\right) \nonumber \\
    &\hspace{0.8cm} +\Delta X_A(k,1)\Delta X_A\left(k-\frac{l+1}{2},0\right) \bigg\} + \frac{\delta_A(\frac{l+1}{2})}{\sqrt{2}}\frac{\delta_A(1)}{\sqrt{2}}  \nonumber \\
    =& \sum_{k=\frac{l+3}{2}}^{\frac{T}{2}} \bigg\{\frac{\delta_A(k)}{\sqrt{2}}\frac{\delta_A(k-\frac{l-1}{2})}{\sqrt{2}} + 
    \frac{\delta_A(k)}{\sqrt{2}}\frac{\delta_A(k-\frac{l+1}{2})}{\sqrt{2}}\bigg\} \nonumber \\
    &\hspace{0.8cm} 
    + \frac{\delta_A(\frac{l+1}{2})}{\sqrt{2}}\frac{\delta_A(1)}{\sqrt{2}}  \nonumber \\
    =& \sum_{k=\frac{l+3}{2}}^{\frac{T}{2}}\frac{\delta_A(k)}{2}\left[\delta_A\left(k-\frac{l-1}{2}\right)+\delta_A\left(k-\frac{l+1}{2}\right)\right]\nonumber\\
    & \hspace{0.5cm} + \frac{\delta_A(\frac{l+1}{2})\delta_A(1)}{2} \nonumber.
\end{align}
Similar to the case for terms (I) and (II), when the lag $l$ is odd, the contributions to the sum for $n=0$ and $n=1$ involve different shifts of $k$ in the perturbed approximation coefficients, namely, $\delta_A\left(k-\frac{l-1}{2}\right)$ and $\delta_A\left(k-\frac{l+1}{2}\right)$ unlike the case when the lag $l$ is even.


In short, term (III) in Equation~\eqref{eqn:terms_case1} can thus be summarized as follows:
\begin{align}
&\text{(III)}= 
\begin{cases} 
 \sum_{k=\frac{l+2}{2}}^{\frac{T}{2}}\delta_A(k) \delta_A\left(k-\frac{l}{2}\right) & \text{if l is even}  \\
 \sum_{k=\frac{l+3}{2}}^{\frac{T}{2}}\frac{\delta_A(k)}{2}\bigg[\delta_A(k-\frac{l-1}{2}) &\text{if l is odd} \\
 \hspace{0.7cm} +\delta_A(k-\frac{l+1}{2})\bigg] + \frac{1}{2} \delta_A(\frac{l+1}{2}) \delta_A(1) & 
 \nonumber
\end{cases}
\end{align}
Combining all the terms (I), (II), and (III), and substituting them back into Equation~(\ref{eqn:terms_case1}) yield the following expression for the perturbed magnitude of autocorrelation when $l$ is even:
\begin{align}
    \rho_{A}(l)-\rho(l) &= \sum_{k=\frac{l+2}{2}}^{\frac{T}{2}} \bigg\{C_k^L \delta_A\left(k-\frac{l}{2}\right) 
    + C^L_{k-\frac{l}{2}}\delta_A(k) \nonumber \\
    & \hspace{1.5cm} + \delta_A(k) \delta_A\left(k-\frac{l}{2}\right) \bigg\}
    \label{eqn:rho_sum_even}
\end{align}
However, when $l$ is odd, the difference in autocorrelation has a more complicated expression:
\begin{align}
    &\rho_{A}(l)-\rho(l)\nonumber \\
    &= \sum_{k=\frac{l+3}{2}}^{\frac{T}{2}}\frac{C_k^L+C_k^H}{2} \delta_A\left(k-\frac{l+1}{2}\right)\nonumber \\ 
    & +\sum_{k=\frac{l+3}{2}}^{\frac{T}{2}}\frac{C_k^L-C_k^H}{2} \delta_A\left(k-\frac{l-1}{2}\right) + \frac{\delta_A(1)}{2}\left(C^L_{\frac{l+1}{2}} -C^H_{\frac{l+1}{2}}\right)  \nonumber \\
    &+ \sum_{k=\frac{l+3}{2}}^{\frac{T}{2}} \frac{C_{k-\frac{l-1}{2}}^L+C_{k-\frac{l-1}{2}}^H}{2}\delta_A(k) \nonumber \\
    & + \sum_{k=\frac{l+3}{2}}^{\frac{T}{2}} \frac{C_{k-\frac{l+1}{2}}^L-C_{k-\frac{l+1}{2}}^H}{2}\delta_A(k) + \frac{\delta_A(\frac{l+1}{2})}{2} \left(C^L_{1} + C^H_{1}\right) \nonumber \\
    & +\sum_{k=\frac{l+3}{2}}^{\frac{T}{2}}\frac{\delta_A(k)}{2}\left[\delta_A(k-\frac{l-1}{2})+\delta_A(k-\frac{l+1}{2})\right] \nonumber\\
    &+ \frac{1}{2} \delta_A\left(\frac{l+1}{2}\right)\delta_A(1)
\label{eqn:rho_sum_odd}
\end{align}

After summing up the absolute values of the difference in autocorrelations in  Equation~\eqref{eqn:rho_sum_even} over the even lags $l = 2m$, where $l = \{0, 2, \cdots, T-2\}$ and $m = \{0, 1, \cdots, \frac{T}{2}-1\}$,  and applying the upper bound $\forall k: \delta_A(k) \le \epsilon^*$, we obtain:
\begin{equation}
    \sum_{m =0}^{\frac{T}{2}-1}|\rho_A(2m)-\rho(2m)| 
	\leq \sum_{m =0}^{\frac{T}{2}-1} \sum_{k=m+1}^{\frac{T}{2}} \epsilon^*\bigg\{|C_k^L| + |C_{k-m}^L| +  \epsilon^*\bigg\}
\end{equation}
where the lower bound for $k$ in Equation~\eqref{eqn:rho_sum_even} becomes $\frac{l+2}{2} = m+1$ after the change of variable. By applying the following change of variable $k' \mapsto k-m$ to the second term on the right-hand side of the inequality above:
$$\sum_{m=0}^{\frac{T}{2}-1} \sum_{k=m+1}^{\frac{T}{2}} |C_{k-m}^L| = \sum_{m=0}^{\frac{T}{2}-1} \sum_{k'=1}^{\frac{T}{2}-m} |C_{k'}^L|$$
and expanding the sum over $m$ for each term, this yields:
\begin{align}
	&\sum_{m =0}^{\frac{T}{2}-1}|\rho_A(2m)-\rho(2m)|
	\nonumber \\
    &\leq \epsilon^* \left[ \sum_{k=1}^{\frac{T}{2}}|C_k^L|+\sum_{k=2}^{\frac{T}{2}}|C_k^L|+\cdots+\sum_{k=\frac{T}{2}-1}^{\frac{T}{2}}|C_k^L|+|C_{\frac{T}{2}}^L|\right]\nonumber \\
    & \ \ \ + \epsilon^* \left[\sum_{k=1}^{\frac{T}{2}}|C_k^L| + \sum_{k=1}^{\frac{T}{2}-1}|C_k^L| + \cdots + \sum_{k=1}^{2}|C_k^L| + |C_1^L|\right]\nonumber \\
    & \ \ \ + \sum_{m =0}^{\frac{T}{2}-1} \left[\frac{T}{2}-m\right] \epsilon^{*2}\nonumber\\
    &= \epsilon^* \sum_{k=1}^{\frac{T}{2}} \left\{ k |C_k^L| +  \left(\frac{T}{2}-k+1\right) |C_k^L|\right\} + \frac{T}{4}\left(\frac{T}{2}+1\right)\epsilon^{*2}\nonumber\\
    &= \epsilon^*\sum_{k=1}^{\frac{T}{2}}\left(\frac{T}{2}+1\right)|C_k^L|  +\frac{T^2+2T}{8}\epsilon^{*2}.
\label{eq:diff1_even}
\end{align}

Next, we repeat the calculations by summing up the absolute values of the difference in autocorrelations in Equation~\eqref{eqn:rho_sum_odd} for the odd lags $l=2m+1$, where $m=\{0,1\cdots, \frac{T}{2}-1\}$ while applying the upper bound $\forall k: |\delta_A(k)| \le \epsilon^*$. This yields:
\begin{align}
    & \sum_{l \in \{1,\cdots.T-1\}} |\rho_A(l)-\rho(l)| \nonumber\\
    & \hspace{0.2cm} = \sum_{m =0}^{\frac{T}{2}-1}|\rho_A(2m+1)-\rho(2m+1)| \nonumber\\ 
    & \hspace{0.2cm} \leq \sum_{m =0}^{\frac{T}{2}-1} \bigg\{
    \sum_{k=m+2}^{\frac{T}{2}} \frac{|C_k^L|+|C_k^H|}{2}\epsilon^* + \sum_{k=m+2}^{\frac{T}{2}} \frac{|C_k^L|+|C_k^H|}{2}\epsilon^*
    \nonumber\\
    & \hspace{1cm} + \frac{|C^L_{m+1}| + |C^H_{m+1}|}{2} \epsilon^* + \sum_{k=m+2}^{\frac{T}{2}} \frac{|C_{k-m}^L|+|C_{k-m}^H|}{2}\epsilon^*\nonumber\\
    & \hspace{1cm} + \sum_{k=m+2}^{\frac{T}{2}} \frac{|C_{k-(m+1)}^L|+|C_{k-(m+1)}^H|}{2}\epsilon^* + \frac{|C^L_{1}| + |C^H_{1}|}{2} \epsilon^*\nonumber\\
    & \hspace{1cm} + \sum_{k=m+2}^{\frac{T}{2}}  \frac{\epsilon^{*}}{2}\left[\epsilon^*+\epsilon^*\right] + \frac{1}{2} \epsilon^{*2}\bigg\}\nonumber\\
    & = \sum_{m =0}^{\frac{T}{2}-1} \epsilon^*\bigg\{\mathcal{C}(m) + \frac{|C^L_{1}| + |C^H_{1}|}{2} + \sum_{k=m+2}^{\frac{T}{2}} \epsilon^* + \frac{1}{2}\epsilon^*\bigg\}\nonumber\\
    & = \sum_{m =0}^{\frac{T}{2}-1} \epsilon^*\bigg\{\mathcal{C}(m) + \frac{|C^L_{1}| + |C^H_{1}|}{2} + \left(\frac{T}{2}-(m+1)\right) \epsilon^* + \frac{1}{2}\epsilon^*\bigg\}\nonumber\\
    & = \sum_{m =0}^{\frac{T}{2}-1} \epsilon^* \mathcal{C}(m) + \frac{T}{4}\left(|C^L_{1}| + |C^H_{1}|\right)\epsilon^*  +  \frac{T}{4}\left( \frac{T}{2}-1 \right) \epsilon^{*2} + \frac{T}{4} \epsilon^{*2}\nonumber\\
    & = \sum_{m =0}^{\frac{T}{2}-1} \epsilon^* \mathcal{C}(m) + \frac{T}{4}\left(|C^L_{1}| + |C^H_{1}|\right)\epsilon^*  +  \frac{T^2}{8}\epsilon^{*2}
    \label{eq:diff1_odd}
\end{align}
where
\begin{align}
\label{eq:common_terms_for_both_cases_when_m}
\mathcal{C}(m)
&=\sum_{k=m+2}^{\frac{T}{2}} (|C_k^L|+|C_k^H|) + \frac{
|C_{m+1}^L|+|C_{m+1}^H|}{2} \\
&+\sum_{k=m+2}^{\frac{T}{2}}
\frac{
|C_{k-m}^L|+|C_{k-m}^H|
+|C_{k-(m+1)}^L|+|C_{k-(m+1)}^H|}{2}\nonumber
\end{align}

\noindent
Summing up the right-hand side of Equations~\eqref{eq:diff1_even} and \eqref{eq:diff1_odd} produces the overall upper bound on the absolute difference in autocorrelation between the approximation-perturbed and original signals:
\begin{align}
    & \sum_{l=0}^{T-1} |\rho_{A}(l)-\rho(l)|\nonumber \\
    \leq& \ \epsilon^*\sum_{k=1}^{\frac{T}{2}} \left(\frac{T}{2}+1\right)|C_k^L| +\sum_{m=0}^{\frac{T}{2}-1}\epsilon^* \mathcal{C}(m) \nonumber\\
    & \ \ \ + \frac{T}{4}\left(|C^L_{1}| + |C^H_{1}|\right)\epsilon^* +\frac{T^2+T}{4}\epsilon^{*2}. 
\label{bound_a}
\end{align}

\newpage
\noindent \textbf{Step 2: Computing Upper Bound of $\sum_{l=0}^{T-1} |\rho_D(l) - \rho(l)|$}  
\vspace{0.2cm}

We now consider the case when the detail coefficients of the Haar wavelet decomposition of the original signal is perturbed. 
Replacing the detail-perturbed signal $X_D'$ from Equation \eqref{eqn:perturbed_D}
$$X_D'(k,n) = X(k,n) + \Delta X_D(k,n)$$
into the definition of autocorrelation function in Equation~\eqref{eqn:autocorr} leads to the following:

\begin{align}
\label{RfD-Rf}
& \rho_{D}(l) - \rho(l) \nonumber\\
&=\sum_n\sum_k \bigg\{\left[X(k,n) + \Delta X_D(k,n)\right] \nonumber\\
& \hspace{1.5cm} \times \left[X\left(k-\frac{l}{2},n\right) + \Delta X_D\left(k-\frac{l}{2},n\right)\right]\nonumber \\
& \hspace{1cm} - \ X(k,n)X\left(k-\frac{l}{2},n\right)\bigg\}
\nonumber\\
&= 
\underbrace{\sum_n\sum_k X(k,n)\,\Delta X_D\left(k-\frac{l}{2},n\right)}_{\text{(I)}} \nonumber \\
& \hspace{1.4cm}  + \underbrace{\sum_n\sum_k \Delta X_D(k,n)\,X\left(k-\frac{l}{2},n\right)}_{\text{(II)}} \nonumber\\
& \hspace{1.4cm} + \underbrace{\sum_n\sum_k \Delta X_D(k,n)\,\Delta X_D\left(k-\frac{l}{2},n\right)}_{\text{(III)}}
\end{align}


\textbf{(1) For term (I)}, we first decompose $X(k,n)$ into $X_A(k,n)+X_D(k,n)$ using Equation~(\ref{proof_decomposition}) to obtain:
\begin{align}
    \text{(I)}&=\sum_n\sum_k X_A(k,n)\Delta X_D\left(k-\frac{l}{2},n\right) \nonumber \\
    &+\sum_n \sum_k X_D(k,n)\Delta X_D\left(k-\frac{l}{2},n\right). \nonumber
\end{align}

When $l$ is even, the first term on the right-hand side of (I) is equal to zero as shown below:
\begin{align}
&\sum_{n=0}^1 \sum_{k=\frac{l+2}{2}}^{\frac{T}{2}} X_A(k,n)\Delta X_D\left(k-\frac{l}{2},n\right)\nonumber \\ 
= & \sum_{k=\frac{l+2}{2}}^{\frac{T}{2}} \bigg\{X_A(k,0)\Delta X_D\left(k-\frac{l}{2},0\right)\nonumber\\
&\hspace{0.8cm} +X_A(k,1)\Delta X_D\left(k-\frac{l}{2},1\right) \bigg\}\nonumber \\
=&\sum_{k=\frac{l+2}{2}}^{\frac{T}{2}} \left\{\frac{C_k^L}{\sqrt{2}}\frac{(-1)^{1}\delta_D(k-\frac{l}{2})}{\sqrt{2}}+
\frac{C_k^L}{\sqrt{2}}\frac{(-1)^{0}\delta_D(k-\frac{l}{2})}{\sqrt{2}}\right\}\nonumber \\
=& \ 0 \nonumber
\end{align}
while the second term on the right-hand side of (I) becomes: 
\begin{align}
    &\sum_{n=0}^1 \sum_{k=\frac{l+2}{2}}^{\frac{T}{2}}X_D(k,n)\Delta X_D\left(k-\frac{l}{2},n\right)\nonumber \\
    =&\sum_{k=\frac{l+2}{2}}^{\frac{T}{2}} \left\{X_D(k,0)\Delta X_D\left(k-\frac{l}{2},0\right) + X_D(k,1) \Delta X_D\left(k-\frac{l}{2},1\right) \right\}\nonumber \\
    =&\sum_{k=\frac{l+2}{2}}^{\frac{T}{2}} \left\{\frac{(-1)^{1}C_k^H}{\sqrt{2}}\frac{(-1)^{1}\delta_D(k-\frac{l}{2})}{\sqrt{2}} 
    + \frac{(-1)^{0}C_k^H}{\sqrt{2}}\frac{(-1)^{0}\delta_D(k-\frac{l}{2})}{\sqrt{2}} \right\}\nonumber \\
    =& \sum_{k=\frac{l+2}{2}}^{\frac{T}{2}}C_k^H \delta_D(k-\frac{l}{2}). \nonumber
\end{align}
When $l$ is odd, the first term on the right-hand side of (I) is 
\begin{align}
    &\sum_{n=0}^1 \sum_{k=\frac{l+3}{2}}^{\frac{T}{2}}X_A(k,n) \Delta X_D\left(k-\frac{l-(-1)^n}{2},1-n\right) \nonumber\\
    &\hspace{1cm} +X_A\left(\frac{l+1}{2},0\right)\Delta X_D(1,1) \nonumber\\
    =& \sum_{k=\frac{l+3}{2}}^{\frac{T}{2}} \bigg\{X_A(k,0)\Delta X_D\left(k-\frac{l-1}{2},1\right) \nonumber\\
    &\hspace{0.8cm} +X_A(k,1)\Delta X_D\left(k-\frac{l+1}{2},0\right) \bigg\} +\frac{C_{\frac{l+1}{2}}^L}{\sqrt{2}}\frac{(-1)^{0}\delta_D(1)}{\sqrt{2}} \nonumber\\
    =& \sum_{k=\frac{l+3}{2}}^{\frac{T}{2}} \bigg\{\frac{C_k^L}{\sqrt{2}}\frac{\delta_D(k-\frac{l-1}{2})}{\sqrt{2}}+\frac{C_k^L}{\sqrt{2}}\frac{(-1)\delta_D(k-\frac{l+1}{2})}{\sqrt{2}}\bigg\}
    +\frac{C_{\frac{l+1}{2}}^L}{\sqrt{2}}\frac{\delta_D(1)}{\sqrt{2}} \nonumber\\
    =& \sum_{k=\frac{l+3}{2}}^{\frac{T}{2}} \frac{C_k^L}{2}\left[\delta_D\left(k-\frac{l-1}{2}\right)-\delta_D\left(k-\frac{l+1}{2}\right)\right]+\frac{C_{\frac{l+1}{2}}^L \delta_D(1)}{2}\nonumber
\end{align}
while the second term on the right-hand side of (I) is
\begin{align}
    &\sum_{n=0}^1 \sum_{k=\frac{l+3}{2}}^{\frac{T}{2}}X_D(k,n) \Delta X_D\left(k-\frac{l-(-1)^n}{2},1-n\right) \nonumber\\
    &\hspace{1cm} +X_D\left(\frac{l+1}{2},0\right)\Delta X_D(1,1) \nonumber\\
    =&\sum_{k=\frac{l+3}{2}}^{\frac{T}{2}} \bigg\{X_D(k,0)\Delta X_D\left(k-\frac{l-1}{2},1\right) \nonumber\\
    & +X_D(k,1)\Delta X_D\left(k-\frac{l+1}{2},0\right)
    \bigg\} +\frac{(-1)^{1}C_{\frac{l+1}{2}}^H}{\sqrt{2}}\frac{\delta_D(1)}{\sqrt{2}}\nonumber \\
    =&\sum_{k=\frac{l+3}{2}}^{\frac{T}{2}} \bigg\{\frac{(-1)^{1}C_k^H}{\sqrt{2}}\frac{(-1)^{0}\delta_D(k-\frac{l-1}{2})}{\sqrt{2}} \nonumber\\
    & \hspace{0.8cm} +\frac{(-1)^{0}C_k^H}{\sqrt{2}}\frac{(-1)^{1}\delta_D(k-\frac{l+1}{2})}{\sqrt{2}} \bigg\}
    - \frac{C_{\frac{l+1}{2}}^H\delta_D(1)}{2} \nonumber\\
    =& -\sum_{k=\frac{l+3}{2}}^{\frac{T}{2}}\frac{C_k^H}{2}\left[\delta_D\left(k-\frac{l-1}{2}\right)+\delta_D\left(k-\frac{l+1}{2}\right)\right] 
    - \frac{C_{\frac{l+1}{2}}^H\delta_D(1)}{2}.\nonumber
\end{align}
Thus, term (I) in Equation~(\ref{RfD-Rf}) can be summarized as follows:
\begin{align}
\text{(I)}=
\begin{cases}
\sum_{k=\frac{l+2}{2}}^{\frac{T}{2}}
C_k^H\,\delta_D\!\left(k-\frac{l}{2}\right)
& \text{if $l$ is even} \\
\sum_{k=\frac{l+3}{2}}^{\frac{T}{2}}
\bigg[
\frac{C_k^L-C_k^H}{2}\,\delta_D\!\left(k-\frac{l-1}{2}\right)
& \text{if $l$ is odd} \\
\hspace{1.2cm} - \frac{C_k^L+C_k^H}{2}\,\delta_D\!\left(k-\frac{l+1}{2}\right)
\bigg]
& \\
\hspace{1.2cm} + \frac{C_{\frac{l+1}{2}}^L-C_{\frac{l+1}{2}}^H}{2}\,\delta_D(1) &
\end{cases}
\nonumber
\end{align}

\textbf{(2) For term (II)}, we split $X(2k-n-l)$ into two parts:
\begin{align}
    \text{(II)}&=\sum_{n}\sum_k \Delta X_D(k,n)X_A\left(k-\frac{l}{2},n\right) \nonumber\\
    &+\sum_{n}\sum_{k} \Delta X_D(k,n)X_D\left(k-\frac{l}{2},n\right) \nonumber.
\end{align}
When the lag $l$ is even, the first term on the right-hand side of (II) vanishes after expanding the sum over $n$ since
\begin{align}
    &\sum_{n=0}^1\sum_{k=\frac{l+2}{2}}^{\frac{T}{2}} \Delta X_D(k,n) X_A\left(k-\frac{l}{2},n\right) \nonumber \\
    =&\sum_{k=\frac{l+2}{2}}^{\frac{T}{2}} \bigg\{\Delta X_D(k,0)X_A\left(k-\frac{l}{2},0\right) \nonumber\\
    &\hspace{1cm} +\Delta X_D(k,1)X_A\left(k-\frac{l}{2},1\right) \bigg\}\nonumber \\
    =& \sum_{k=\frac{l+2}{2}}^{\frac{T}{2}}\frac{(-1)^{1}\delta_D(k)}{\sqrt{2}}\frac{C_{k-\frac{l}{2}}^L}{\sqrt{2}} + \sum_{k=\frac{l+2}{2}}^{\frac{T}{2}}\frac{(-1)^{0}\delta_D(k)}{\sqrt{2}}\frac{C_{k-\frac{l}{2}}^L}{\sqrt{2}}\nonumber \\
    =& \ 0 \nonumber
\end{align}
while the second term on the right-side of (II) becomes
\begin{align}
    &\sum_{n=0}^1\sum_{k=\frac{l+2}{2}}^{\frac{T}{2}} \Delta X_D(k,n) X_D\left(k-\frac{l}{2},n\right)\nonumber \\
    =&\sum_{k=\frac{l+2}{2}}^{\frac{T}{2}} \bigg\{\Delta X_D(k,0) X_D\left(k-\frac{l}{2},0\right) \nonumber\\
    & \hspace{1cm} + \Delta X_D(k,1) X_D\left(k-\frac{l}{2},1\right)\bigg\}\nonumber \\
    =&\sum_{k=\frac{l+2}{2}}^{\frac{T}{2}} \bigg\{\frac{(-1)\delta_D(k)}{\sqrt{2}} \frac{(-1)C_{k-\frac{l}{2}}^H}{\sqrt{2}} 
    + \frac{\delta_D(k)}{\sqrt{2}} \frac{C_{k-\frac{l}{2}}^H}{\sqrt{2}} \bigg\}\nonumber \\
    =& \sum_{k=\frac{l+2}{2}}^{\frac{T}{2}}C_{k-\frac{l}{2}}^H \delta_D(k)\nonumber.
\end{align}

Next, when $l$ is odd, the first term on the right-hand side of (II) can be written as follows:
\begin{align}
    &\sum_{n=0}^1\sum_{k=\frac{l+3}{2}}^{\frac{T}{2}} \Delta X_D(k,n) \ X_A\left(k-\frac{l-(-1)^n}{2},1-n\right) \nonumber \\
    &\hspace{1cm} +\Delta X_D\left(\frac{l+1}{2},0\right) X_A(1,l) \nonumber \\
    =& \sum_{k=\frac{l+3}{2}}^{\frac{T}{2}} \bigg\{\Delta X_D(k,0)X_A\left(k-\frac{l-1}{2},1\right) \nonumber \\
    &\hspace{0.5cm} + \Delta X_D(k,1)X_A\left(k-\frac{l+1}{2},0\right) \bigg\} +\frac{(-1)^{1}\delta_D(\frac{l+1}{2})}{\sqrt{2}}\frac{C_1^L}{\sqrt{2}}\nonumber \\
    =& \sum_{k=\frac{l+3}{2}}^{\frac{T}{2}} \bigg\{\frac{(-1)^{1}\delta_D(k)}{\sqrt{2}}\frac{C_{k-\frac{l-1}{2}}^L}{\sqrt{2}} + \frac{(-1)^{0}\delta_D(k)}{\sqrt{2}}\frac{C_{k-\frac{l+1}{2}}^L}{\sqrt{2}}\bigg\}\nonumber \\
    &\hspace{0.5cm} -\frac{C_1^L \delta_D(\frac{l+1}{2})}{2}\nonumber \\
    =&\sum_{k=\frac{l+3}{2}}^{\frac{T}{2}}\frac{\delta_D(k)}{2}(C_{k-\frac{l+1}{2}}^L-C_{k-\frac{l-1}{2}}^L)-\frac{C_1^L \delta_D(\frac{l+1}{2})}{2} \nonumber
\end{align}
while the second term in (II) can be simplified as
\begin{align}
    &\sum_{n=0}^1\sum_{k=\frac{l+3}{2}}^{\frac{T}{2}} \Delta X_D(k,n)X_D\left(k-\frac{l-(-1)^n}{2},1-n\right) \nonumber \\
    &\hspace{1cm} +\Delta X_D\left(\frac{l+1}{2},0\right)X_D(1,1) \nonumber \\ \nonumber
    =&\sum_{k=\frac{l+3}{2}}^{\frac{T}{2}} \bigg\{\Delta X_D(k,0)X_D\left(k-\frac{l-1}{2},1\right) \nonumber\\
    &\hspace{0.1cm} + \Delta X_D(k,1)X_D\left(k-\frac{l+1}{2},0\right)\bigg\}
    + \frac{(-1)^{1}\delta_D(\frac{l+1}{2})}{\sqrt{2}}\frac{(-1)^{0}C_1^H}{\sqrt{2}} \nonumber \\
    =& \sum_{k=\frac{l+3}{2}}^{\frac{T}{2}} \bigg\{\frac{(-1)^{1}\delta_D(k)}{\sqrt{2}}\frac{(-1)^{0}C_{k-\frac{l-1}{2}}^H}{\sqrt{2}} \nonumber \\
    &\hspace{1cm} + \frac{(-1)^{0}\delta_D(k)}{\sqrt{2}}\frac{(-1)^{1}C_{k-\frac{l+1}{2}}^H}{\sqrt{2}} \bigg\} - \frac{\delta_D(\frac{l+1}{2})C_1^H}{2} \nonumber \\
    =&-\sum_{k=\frac{l+3}{2}}^{\frac{T}{2}}\frac{\delta_D(k)}{2}(C_{k-\frac{l-1}{2}}^H+C_{k-\frac{l+1}{2}}^H)- \frac{\delta_D(\frac{l+1}{2})C_1^H}{2}\nonumber.
\end{align}

Combining the terms together, the second term in Equation~(\ref{RfD-Rf}) can be summarized as follows:
\begin{align}
\text{(II)}=
\begin{cases}
\sum_{k=\frac{l+2}{2}}^{\frac{T}{2}}
C_{k-\frac{l}{2}}^H\,\delta_D(k),
& \text{if $l$ is even} \\
\sum_{k=\frac{l+3}{2}}^{\frac{T}{2}}
\delta_D(k) \bigg[
\frac{C_{k-\frac{l+1}{2}}^L-C_{k-\frac{l+1}{2}}^H}{2}
& \text{if $l$ is odd}\\
\hspace{0.4cm} -\frac{C_{k-\frac{l-1}{2}}^L+C_{k-\frac{l-1}{2}}^H}{2}
\bigg] -\frac{C_1^L+C_1^H}{2}\,\delta_D\!\left(\frac{l+1}{2}\right)
& 
\end{cases}
\nonumber
\end{align}

\textbf{(3) For term (III)}, the expression can be simplified as follows when $l$ is even: 
\begin{align}
    &\sum_{n=0}^1\sum_{k=\frac{l+2}{2}}^{\frac{T}{2}}\Delta X_D(k,n) \Delta X_D\left(k-\frac{l}{2},n\right) \nonumber \\
    =&\sum_{k=\frac{l+2}{2}}^{\frac{T}{2}} \bigg\{\Delta X_D(k,0) \Delta X_D\left(k-\frac{l}{2},0\right) \nonumber\\
    &\hspace{1cm} + \Delta X_D(k,1)\Delta X_D\left(k-\frac{l}{2},1\right)\bigg\}\nonumber\\
    =& \sum_{k=\frac{l+2}{2}}^{\frac{T}{2}} \bigg\{\frac{(-1)^{1}\delta_D(k)}{\sqrt{2}}\frac{(-1)^{1}\delta_D(k-\frac{l}{2})}{\sqrt{2}} \nonumber \\
    &\hspace{1cm} + 
    \frac{(-1)^{0}\delta_D(k)}{\sqrt{2}}\frac{(-1)^{0}\delta_D(k-\frac{l}{2})}{\sqrt{2}}\bigg\}\nonumber \\
    =& \sum_{k=\frac{l+2}{2}}^{\frac{T}{2}}\delta_D(k)\delta_D\left(k-\frac{l}{2}\right)\nonumber
\end{align}
and to the following when $l$ is odd:
\begin{align}
   &\sum_{n=0}^1 \sum_{k=\frac{l+3}{2}}^{\frac{T}{2}} X_D(k,n) \Delta X_D\left(k-\frac{l-(-1)^n}{2},1-n\right)  \nonumber\\
   &\hspace{1cm} +\Delta X_D\left(\frac{l+1}{2},0\right)\Delta X_D(1,l) \nonumber \\
   =& \sum_{k=\frac{l+3}{2}}^{\frac{T}{2}} \bigg\{\Delta X_D (k,0) \Delta X_D\left(k-\frac{l-1}{2},1\right) \nonumber \\
   &\hspace{1cm} + \Delta X_D(k,1) \Delta X_D\left(k-\frac{l+1}{2},0\right) \bigg\}\nonumber\\
   &\hspace{1cm} +\frac{(-1)^{1}\delta_D(\frac{l+1}{2})}{\sqrt{2}} \frac{(-1)^{0}\delta_D(1)}{\sqrt{2}} \nonumber \\
   =& \sum_{k=\frac{l+3}{2}}^{\frac{T}{2}} \bigg\{\frac{(-1)^{1}\delta_D(k)}{\sqrt{2}}\frac{(-1)^{0}\delta(k-\frac{l-1}{2})}{\sqrt{2}} \nonumber \\
   &\hspace{0.8cm} + \frac{(-1)^{0}\delta_D(k)}{\sqrt{2}}\frac{(-1)^{1}\delta_D(k-\frac{l+1}{2})}{\sqrt{2}} -\frac{\delta_D(\frac{l+1}{2})\delta_D(1)}{2} \nonumber \\
   =&-\sum_{k=\frac{l+3}{2}}^{\frac{T}{2}} \frac{\delta_D(k-\frac{l-1}{2})+\delta_D(k-\frac{l+1}{2})}{2}\delta_D(k)-\frac{\delta_D(\frac{l+1}{2})\delta_D(1)}{2} \nonumber 
\end{align}

In summary, term (III) in Equation~(\ref{RfD-Rf}) is given by
\begin{align}
\text{(III)}=
\begin{cases}
\sum_{k=\frac{l+2}{2}}^{\frac{T}{2}}
\delta_D(k)\,\delta_D\!\left(k-\frac{l}{2}\right),
& \text{if $l$ is even} \\
-\frac{1}{2}
\sum_{k=\frac{l+3}{2}}^{\frac{T}{2}}
\delta_D(k) \bigg[
\,\delta_D\!\left(k-\frac{l-1}{2}\right) & \text{if $l$ is odd}
\\
\hspace{0.8cm} + \delta_D\!\left(k-\frac{l+1}{2}\right)
\bigg]-\frac{\delta_D\!\left(\frac{l+1}{2}\right)\delta_D(1)}{2}
&
\end{cases}
\nonumber
\end{align}


\noindent Substituting the three terms into Equation~(\ref{RfD-Rf}), we obtain the following when the lag $l$ is even:
\begin{align}
\rho_D(l)-\rho(l) =& \sum_{k=\frac{l+2}{2}}^{\frac{T}{2}}
\bigg\{ C_k^H\,\delta_D\!\left(k-\frac{l}{2}\right)
+ C_{k-\frac{l}{2}}^H\,\delta_D(k) \nonumber\\
&\hspace{1cm} + \delta_D(k)\,\delta_D\!\left(k-\frac{l}{2}\right)\bigg\}
\end{align}
Taking the absolute value of their difference in autocorrelation and using the assumption that $\forall k: \delta_D(k) \le \epsilon^*$, we obtain the following inequality for the even lags, $l = 2m$, where $m=\{0,1,\cdots,\frac{T}{2}-1\}$:
\begin{align}
&|\rho_D(2m)-\rho(2m)|
\le \epsilon^* \sum_{k=m+1}^{\frac{T}{2}} \bigg\{|C_k^H|
+|C_{k-m}^H| + \epsilon^*\bigg\} \nonumber\\
& \hspace{1cm}  = \epsilon^* \bigg\{ \sum_{k=m+1}^{\frac{T}{2}} |C_k^H| 
+ \sum_{k=1}^{\frac{T}{2}-m} |C_{k}^H| + \left(\frac{T}{2}-m\right)\epsilon^{*}\bigg\} \nonumber
\end{align}
after changing the variable $k-m \mapsto k$ in the second term on the last line. Upon summing over the even lags $m$, this yields:
\begin{align}
	&\sum_{m =0}^{\frac{T}{2}-1}|\rho_D(2m)-\rho(2m)|
	\nonumber \\
    &\leq \epsilon^* \left[ \sum_{k=1}^{\frac{T}{2}}|C_k^H|+\sum_{k=2}^{\frac{T}{2}}|C_k^H|+\cdots+\sum_{k=\frac{T}{2}-1}^{\frac{T}{2}}|C_k^H|+|C_{\frac{T}{2}}^H|\right]\nonumber \\
    & \ \ \ + \epsilon^* \left[\sum_{k=1}^{\frac{T}{2}}|C_k^H| + \sum_{k=1}^{\frac{T}{2}-1}|C_k^H| + \cdots + \sum_{k=1}^{2}|C_k^H| + |C_k^H|\right]\nonumber \\
    & \ \ \ + \sum_{m =0}^{\frac{T}{2}-1} \left[\frac{T}{2}-m\right] \epsilon^{*2}\nonumber\\
    &= \epsilon^* \sum_{k=1}^{\frac{T}{2}} \left\{ k |C_k^H| +  \left(\frac{T}{2}-k+1\right) |C_k^H|\right\} + \frac{T}{4}\left(\frac{T}{2}+1\right)\epsilon^{*2}\nonumber\\
    &= \epsilon^*\sum_{k=1}^{\frac{T}{2}}\left(\frac{T}{2}+1\right)|C_k^H|  +\frac{T^2+2T}{8}\epsilon^{*2}.
\label{eq:diff2_even}
\end{align}

In contrast, when the lag $l$ is odd, we obtain

{\small
\begin{align}
&\rho_D(l)-\rho(l)\nonumber\\
=&\sum_{k=\frac{l+3}{2}}^{\frac{T}{2}}
\left[
\frac{C_k^L-C_k^H}{2}\,\delta_D\!\left(k-\frac{l-1}{2}\right)
-
\frac{C_k^L+C_k^H}{2}\,\delta_D\!\left(k-\frac{l+1}{2}\right)
\right] \nonumber\\
& \quad +\sum_{k=\frac{l+3}{2}}^{\frac{T}{2}}
\left[
\frac{C_{k-\frac{l+1}{2}}^L-C_{k-\frac{l+1}{2}}^H}{2}
-
\frac{C_{k-\frac{l-1}{2}}^L+C_{k-\frac{l-1}{2}}^H}{2}
\right]\delta_D(k) \nonumber\\
& \quad  +\frac{C_{\frac{l+1}{2}}^L-C_{\frac{l+1}{2}}^H}{2}\,\delta_D(1) -\frac{C_1^L+C_1^H}{2}\,\delta_D\!\left(\frac{l+1}{2}\right) \nonumber\\
& \quad -\frac{1}{2}
\sum_{k=\frac{l+3}{2}}^{\frac{T}{2}}
\left[
\delta_D(k)\,\delta_D\!\left(k-\frac{l-1}{2}\right)
+
\delta_D(k)\,\delta_D\!\left(k-\frac{l+1}{2}\right)
\right] \nonumber\\
& \quad -\frac{1}{2}\delta_D\!\left(\frac{l+1}{2}\right)\delta_D(1).
\end{align}
}

Taking the absolute value of the difference in autocorrelation and summing them over the odd lags $l = 2m+1$, where $m=\{0,1,\cdots,\frac{T}{2}-1$\}, 
the following inequality holds given the assumption that $\forall k: |\delta_D(k)| \le \epsilon^*$:
\begin{align}
    & \sum_{l \in \{1,\cdots.T-1\}} |\rho_D(l)-\rho(l)| = \sum_{m =0}^{\frac{T}{2}-1}|\rho_D(2m+1)-\rho(2m+1)| \nonumber\\ 
    & \leq \sum_{m =0}^{\frac{T}{2}-1} \epsilon^*\bigg\{\mathcal{C}(m) + \frac{|C^L_{1}| + |C^H_{1}|}{2} + \sum_{k=m+2}^{\frac{T}{2}} \epsilon^* + \frac{1}{2}\epsilon^*\bigg\} \nonumber\\
     & = \sum_{m=0}^{\frac{T}{2}-1} \mathcal{C}(m) \epsilon^* + \frac{T}{4} \left(|C_1^L| + |C_1^H|\right) \epsilon^* \nonumber\\
     &\hspace{1cm} + \frac{T}{4}\left( \frac{T}{2}-1 \right) \epsilon^{*2} + \frac{T}{4} \epsilon^{*2} \nonumber \\
    & = \sum_{m=0}^{\frac{T}{2}-1}\epsilon^* \mathcal{C}(m) + \frac{T}{4} \left(|C_1^L| + |C_1^H|\right) \epsilon^* + \frac{T^2}{8} \epsilon^{*2}
\label{eq:rho_A-rho_sum_over_odd}
\end{align}
using the definition of $C(m)$ given in Equation~\eqref{eq:common_terms_for_both_cases_when_m}. 
Combining the two inequalities given in \eqref{eq:diff2_even} and \eqref{eq:rho_A-rho_sum_over_odd}, we obtain
\begin{align}
    & \sum_{l=0}^{T-1} |\rho_{D}(l)-\rho(l)|
    \leq \ \epsilon^*\sum_{k=1}^{\frac{T}{2}} \left(\frac{T}{2}+1\right)|C_k^H| +\sum_{m=0}^{\frac{T}{2}-1}\epsilon^* \mathcal{C}(m) \nonumber\\
    & \hspace{2.6cm} + \frac{T}{4} \left(|C_1^L| + |C_1^H|\right) \epsilon^* +\frac{T^2+T}{4}\epsilon^{*2} 
    \label{bound_d}
\end{align}

\vspace{0.2cm}
\noindent \textbf{Step 3: Upper bounds comparison}. 

Based on the resulting inequalities in (\ref{bound_a}) and~(\ref{bound_d}), the upper bounds on the difference in autocorrelation when perturbing the approximation and detail coefficients are as follows:
\begin{align*}
\sup_{\delta_A} \sum_{l=0}^{T-1} |\rho_{A}(l)-\rho(l)| &= 
\epsilon^*\sum_{k=1}^{\frac{T}{2}} \left(\frac{T}{2}+1\right)|C_k^L| +\sum_{m=0}^{\frac{T}{2}-1}\epsilon^* \mathcal{C}(m) \nonumber\\
    & \ + \frac{T}{4}\left(|C^L_{1}| + |C^H_{1}|\right)\epsilon^* +\frac{T^2+T}{4}\epsilon^{*2} \\
\sup_{\delta_D} \sum_{l=0}^{T-1} |\rho_{D}(l)-\rho(l)| &= 
\epsilon^*\sum_{k=1}^{\frac{T}{2}} \left(\frac{T}{2}+1\right)|C_k^H| +\sum_{m=0}^{\frac{T}{2}-1}\epsilon^* \mathcal{C}(m)\nonumber\\
& \  + \frac{T}{4} \left(|C_1^L| + |C_1^H|\right) \epsilon^* +\frac{T^2+T}{4}\epsilon^{*2} 
\end{align*}
Observe that both upper-bound expressions share three common terms, differing only in their first term. Thus, assuming
\begin{align*}
    \sum_{k=1}^{\frac{T}{2}}|C_k^H| \leq \sum_{k=1}^{\frac{T}{2}}|C_k^L|, \nonumber
\end{align*}
it follows that
\[
\sup_{\delta_D} \left( \sum_{l=0}^{T-1} |\rho_{D}(l) - \rho(l)| \right) \leq \sup_{\delta_A} \left( \sum_{l=0}^{T-1} |\rho_{A}(l) - \rho(l)| \right).
\]
which completes the proof.
\end{proof}

\section{Proof of Theorem~\ref{thm:closeness}}
\label{appendix:Proof of Theorem closeness}

\begin{theorem}
Let $\mathbf{X}_A'$ and $\mathbf{X}_D'$ be the signals obtained by perturbing only the approximation and detail coefficients of $\mathbf{X}$, respectively.
If $\|\frac{1}{\sqrt{2}}\boldsymbol{\delta}_D\|_2 \leq \|\frac{1}{\sqrt{2}}\boldsymbol{\delta}_A\|_2$, then
$
\|\mathbf{X}_D' - \mathbf{X}\|_2 \leq \|\mathbf{X}_A' - \mathbf{X}\|_2.
$
\end{theorem}


\begin{proof}
Consider a one-dimensional signal ${X(t)}$ of length $T$ with the following level-one Haar wavelet decomposition: 
\begin{align}
X(2k - n) &= \frac{C_k^L}{\sqrt{2}} + \frac{(-1)^{1-n} C_k^H}{\sqrt{2}}, \nonumber
\end{align}
where \( n \in \{0,1\} \), and \( \displaystyle k\in\{1,2,\dots,\frac{T}{2}\} \). Define the vectors
$$\mathbf{X} = \begin{pmatrix}
    X(1) \\ X(2) \\ \cdots \\ X(2k)
\end{pmatrix} \in \mathbb{R}^T \quad \textrm{and} \quad
\mathbf{C} = \begin{pmatrix}
C_1^L\\
C_2^L \\
...\\
C_k^L\\
C_1^H\\
C_2^H\\
...\\
C_k^H
\end{pmatrix} \in \mathbb{R}^T
$$
and the following orthogonal Haar decomposition matrix
$$\mathbf{W} = \frac{1}{\sqrt{2}} \begin{pmatrix}
1 & 0 & 0 & \cdots & 0 & 1 & 0 & 0 & \cdots & 0 \\
1 & 0 & 0 & \cdots & 0 & -1 & 0 & 0 & \cdots & 0 \\
0 & 1 & 0 & \cdots & 0 & 0 & 1 & 0 & \cdots & 0 \\
0 & 1 & 0 & \cdots & 0 & 0 & -1 & 0 & \cdots & 0 \\
\vdots & \vdots & \vdots & \ddots & \vdots & \vdots & \vdots & \vdots & \ddots & \vdots \\
0 & 0 & 0 & \cdots & 1 & 0 & 0 & 0 & \cdots & 1 \\
0 & 0 & 0 & \cdots & 1 & 0 & 0 & 0 & \cdots & -1 \\
\end{pmatrix}$$
Furthermore, it can be shown that
$$\mathbf{X} = \mathbf{WC}$$
%
For brevity, we further denote $\mathbf{C} = [\mathbf{a}_0; \mathbf{d}_0]^\top$, where $\mathbf{a}_0 \in \mathbb{R}^{T/2}$ and $\mathbf{d}_0 \in \mathbb{R}^{T/2}$ are the respective approximate and detail coefficients.

Therefore, the perturbed signals can be expressed as
$$
\mathbf{X}_A' = \mathbf{W}\left[ \mathbf{a}_0 + \mathbf{\delta_{A}}; \; \mathbf{d}_0 \right]^\top = \mathbf{X} + \mathbf{W} [\mathbf{\delta_{A}}; \mathbf{0}]^\top,
$$
and
$$
\mathbf{X}_D' = W \left[ \mathbf{a}_0; \; \mathbf{d}_0 + \mathbf{\delta_{D}} \right]^\top = \mathbf{X} + W [\mathbf{0}; \mathbf{\delta_{D}}]^\top.
$$
Since $\mathbf{W}$ is an orthogonal matrix, $\mathbf{W}^\top\mathbf{W} = \mathbf{WW}^\top = \mathbf{I}$. Therefore:
$$
\|\mathbf{X}_A'-\mathbf{X}\|_2=\|\mathbf{W} [\mathbf{\delta_{A}}; \mathbf{0}]\|_2=\|\frac{1}{\sqrt{2}}\mathbf{\delta_{A}}\|_2,
$$
and
$$
\|\mathbf{X}_D'-\mathbf{X}\|_2=\|\mathbf{W} [ \mathbf{0};\mathbf{\delta_D}]\|_2=\|\frac{1}{\sqrt{2}}\mathbf{\delta_D}\|_2.
$$
Assuming 
$$
\|\frac{1}{\sqrt{2}}\boldsymbol{\delta}_D\|_2 \leq  \|\frac{1}{\sqrt{2}}\boldsymbol{\delta}_A\|_2 ,
$$
this implies that
$$
\|\mathbf{X}_D' - \mathbf{X}\|_2 \leq  \|\mathbf{X}_A' - \mathbf{X}\|_2 .
$$
\end{proof}

\vfill

\end{document}